\documentclass[sigconf]{acmart}
\usepackage{multicol}
\usepackage{algorithm}
\usepackage{euscript}
\usepackage{mathtools}
\usepackage[noend]{algorithmic}
\usepackage{array}
\usepackage{amsmath,caption,booktabs}

\graphicspath{{IpeFigures/}{Figures/}}
\newcommand{\Paragraph}[1]{\vspace{1mm} \noindent \sffamily \textbf{#1}.}

\newcommand{\Eu}[1]{\ensuremath{\EuScript{#1}}}

\newcommand{\E}{\ensuremath{\mathsf{E}}}
\newcommand{\R}{\ensuremath{\mathbb{R}}}
\newcommand{\D}{\ensuremath{\Eu{D}}}
\newcommand{\K}{\ensuremath{\Eu{K}}}
\newcommand{\RR}{\ensuremath{\Eu{R}}}

\newcommand{\eps}{\varepsilon}

\newcommand{\argmax}{\ensuremath{\mathrm{argmax}}}
\newcommand{\argmin}{\ensuremath{\mathrm{argmin}}}
\newcommand{\dir}[1]{\ensuremath{\mathsf{d}#1}}
\newcommand{\Ber}{\textsf{Be}}
\newcommand{\Poi}{\textsf{P}}
\newcommand{\Gau}{\textsf{G}}

\AtBeginDocument{%
  \providecommand\BibTeX{{%
    \normalfont B\kern-0.5em{\scshape i\kern-0.25em b}\kern-0.8em\TeX}}}

\begin{document}

\title{The Kernel Spatial Scan Statistic}

\author{Mingxuan Han}
\authornote{Both authors contributed equally to this research.}
\email{u1209601@umail.utah.edu}
\affiliation{%
  \institution{University of Utah}
  \streetaddress{201 Presidents Circle}
  \postcode{84112}
}
\author{Michael Matheny}
\authornotemark[1]
\email{mmath@cs.utah.edu}
\affiliation{%
  \institution{University of Utah}
  \streetaddress{201 Presidents Circle}
  \postcode{84112}
}

\author{Jeff M. Phillips}
\email{jeffp@cs.utah.edu}
\affiliation{%
 \institution{University of Utah}
  \streetaddress{201 Presidents Circle}
  \postcode{84112}
}



\begin{abstract}

Kulldorff's (1997) seminal paper on spatial scan statistics (SSS) has led to many methods considering different regions of interest, different statistical models, and different approximations while also having numerous applications in epidemiology, environmental monitoring, and homeland security. SSS provides a way to rigorously test for the existence of an anomaly and provide statistical guarantees as to how "anomalous" that anomaly is. However, these methods rely on defining specific regions where the spatial information a point contributes is limited to  binary 0 or 1, of either inside or outside the region, while in reality anomalies will tend to follow smooth distributions with decaying density further from an epicenter.  
%
In this work, we propose a method that addresses this shortcoming through a continuous scan statistic that generalizes SSS by allowing the point contribution to be defined by a kernel.  We provide extensive experimental and theoretical results that shows our methods can be computed efficiently while providing high statistical power for detecting anomalous regions.

\end{abstract}




\maketitle

\section{Introduction}
We propose a generalized version of spatial scan statistics called the \emph{kernel spatial scan statistic}.  
In contrast to the many variants~\cite{AMPVW06, Kul97, huang2007spatial, jung2007spatial, neill2006bayesian, jung2010spatial, kulldorff2009scan, FNN17,HMWN18, SSMN16} of this classic method for geographic information sciences, the kernel version allows for the modeling of a gradually dampening of an anomalous event as a data point becomes further from the epicenter.  As we will see, this modeling change allows for more statistical power \emph{and} faster algorithms (independent of data set size), in addition to the more realistic modeling.  

\begin{figure}  
\vspace{-1mm}
\includegraphics[width=\linewidth]{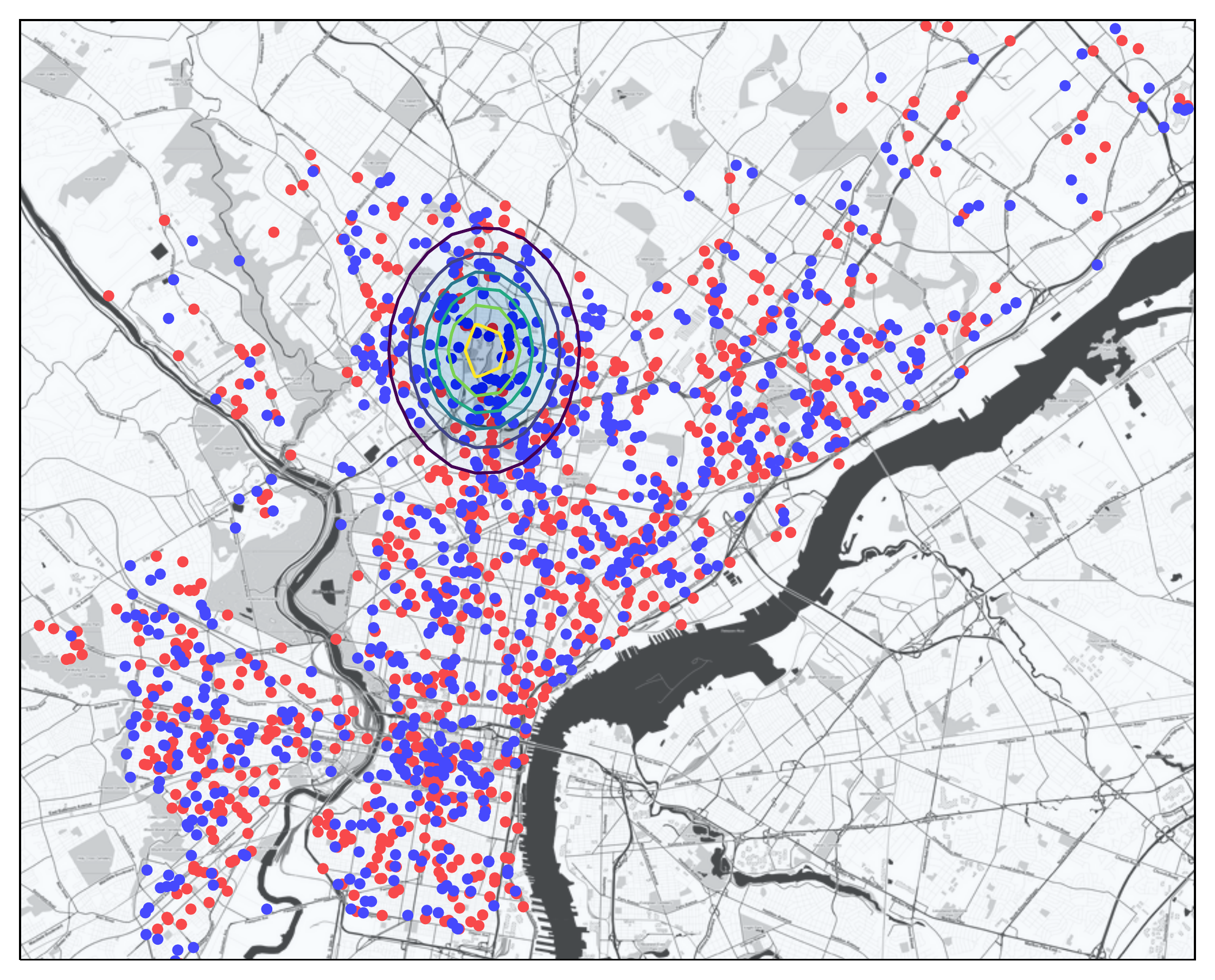}

\vspace{-3mm}
\caption{\label{fig:Bernoulli-model}
An anomaly affecting $8\%$ of data (and rate parameters $p=.9$, $q=.5$) under the Bernoulli Kernel Spatial Scan Statistic model on geo-locations of crimes in Philadelphia.}
\vspace{-2mm}
\end{figure}

To review, spatial scan statistics consider a baseline or population data set $B$ where each point $x \in B$ has an annotated value $m(x)$.  In the simplest case, this value is binary ($1$ = person has cancer; $0$ = person does not have cancer), and it is useful to define the measured set $M \subset B$ as $M = \{x \in X \mid m(x) = 1\}$.  
Then the typical goal is to identify a region where there are significantly more measured points than one would expect from the baseline data $B$.  
To prevent overfitting (e.g., gerrymandering), the typical formulation fixes a set of potential anomalous regions $\RR$ induced by a family of geometric shapes: disks~\cite{Kul97}, axis-aligned rectangles~\cite{NM04}, ellipses~\cite{Kul7.0}, halfspaces~\cite{MP18a}, and others~\cite{Kul7.0}.  
Then given a statistical discrepancy function $\Phi(R) = \phi(R(M),R(B))$, where typically $R(M) = \frac{|R \cap M|}{|M|}$ and $R(B) = \frac{|R \cap B|}{|B|}$, the spatial scan statistic SSS is 
\[
\Phi^*  = \max_{R \in \RR} \Phi(R).
\]
And the hypothesized anomalous region is $R^* = \argmax_{R \in \RR} \Phi(R)$ so $\Phi^* = \Phi(R^*)$.  
Conveniently, by choosing a fixed set of shapes, and having a fixed baseline set $B$, this actually combinatorially limits the set of all possible regions that can be considered (when computing the $\max$ operator), since for instance, there can be only $O(|B|^3)$ distinct disks which each contains a different subset of points.  This allows for tractable~\cite{AMPVW06} (and in some cases very scalable~\cite{MP18b}) combinatorial and statistical algorithms which can (approximately) search over the class of \emph{all} shapes from that family.  
Alternatively, the most popular software, SatScan~\cite{Kul7.0} uses a fixed center set of possible epicenters of events, for simpler and more scalable algorithms.  

However, the discreteness of these models has a strange modeling side effect.  Consider a the shape model of disks $\D$, where each disk $D \in \D$ is defined $D = \{x \in \R^d \mid \|x-c\| \leq r\}$ by a center $c \in \R^d$ and a radius $r > 0$.  Then solving a spatial scan statistic over this family $\D$ would yield an anomaly defined by a disk $D$; that is, all points $x \in B \cap D$ are counted entirely inside the anomalous region, and all points $x' \in B \setminus D$ are considered entirely outside the anomalous region.  
If this region is modeling a regional event; say the area around a potentially hazardous chemical leak suspected of causing cancer in nearby residents, then the hope is that the center $c$ identifies the location of the leak, and $r$ determines the radius of impact.  
However, this implies that data points $x \in B$ very close to the epicenter $c$ are affected equally likely as those a distance of almost but not quite $r$ away.  And those data points $x' \in B$ that are slightly further than $r$ away from $c$ are not affected at all.  
In reality, the data points closest to the epicenter should be more likely to be affected than those further away, even if they are within some radius $r$, and data points just beyond some radius should still have some, but a lessened effect as well.

\Paragraph{Introducing the Kernel Spatial Scan Statistic}
The main modeling change of the kernel spatial scan statistic (KSSS) is to prescribe these diminishing effects of spatial anomalies as data points become further from the epicenter of a proposed event.  From a modeling perspective, given the way we described the problem above, the generalization is quite natural: we simply replace the shape class $\RR$ (e.g., the family of all disks $\D$) with a class of non-binary continuous functions $\K$.  The most natural choice (which we focus on) are kernels, and in particular Gaussian kernels.  We define each $K \in \K$ by a center $c$ and a bandwidth $r$ as $K(x) = \exp(-\|x-c\|^2/r^2).$  This provides a real value $K(x) \in [0,1]$, in fact a probability, for each $x \in B$.  
We interpret this as: given an anomaly model $K$, then for each $x \in B$ the value $K(x)$ is the probability that the rate of the measured event (chance that a person gets cancer) is increased.  

\Paragraph{Related Work on SSS and Kernels} There have been many papers on computing various range spaces for SSS \cite{NM04,Tango2005,ACTW18, ICDMS2015} where a geometric region defines the set of points that included in the region for various regions such as disks, ellipses, rings, and rectangles. Other work has combined SSS and kernels as a way to penalize far away points, but still used binary regions, and only over a set of predefined starting points~\cite{DA04, DCTB2007, Patil2004}.   
Another method~\cite{FNN17} uses a Kernel SVM boundary to define a region; this provides a regularized, but otherwise very flexible class of regions -- but they are still binary. 
A third method~\cite{JRSA89}, proposes an inhomogeneous Poisson process model for the spatial relationship between a measured cancer rate and exposure to single specified region (from industrial pollution source). 
This models the measured rate similar to our work, but does not search over a family of regions, and does not model a background rate.  


\Paragraph{Our contributions, and their challenges}
We formally derive and discuss in more depth the KSSS in Section \ref{sec:model}, and contrast this with related work on SSS. 
While the above intuition is (we believe) quite natural, and seems rather direct, a complication arises: the contribution of the data points towards the statistical discrepancy function (derived as a log-likelihood ratio) are no longer independent.  This implies that $K(B)$ and $K(M)$ can no longer in general be scalar values (as they were with $R(B)$ and $R(M)$); instead we need to pass in sets.  Moreover, this means that unlike with traditional binary ranges, the value of $\Phi$ no longer in general has a closed form; in particular the optimal rate parameters in the alternative hypothesis do not have a closed form.  We circumvent this by describing a simple convex cost function for the rate parameters.  And it turns out, these can then be effectively solved for with a few steps of gradient descent for each potential choice of $K$ within the main scanning algorithm.  

Our paper then focuses on the most intuitive Bernoulli model for how measured values are generated, but the procedures we derive will apply similar to the Poisson and Gaussian models we also derive.  For instance, it turns out that the Gaussian model kernel statistical discrepancy function has a closed form.

The second major challenge is that there is no longer a combinatorial limit on the number of distinct ranges to consider.  There are an infinite set of potential centers $c \in \R^d$ to consider, even with a fixed bandwidth, and each could correspond to a different $\Phi(K)$ value.  However, there is a Lipschitz property on $\Phi(K)$ as a function of the choice of center $c$; that is if we change $c$ to $c'$ by a small amount, then we can upperbound the change in $\Phi(K_c)$ to $\Phi(K_{c'})$ by a linear function of $\|c-c'\|$.  
This implies a finite resolution needed to consider on the set of center points: we can lay down a fixed resolution grid, and only consider those grid points.  Notably: \emph{this property does not hold for the combinatorial SSS version}, as a direct effect of the problematic boundary issue of the binary ranges.  

We combine the insight of this Lipschitz property, and the gradient descent to evaluate $\Phi(K_c)$ for a set of center points $c$, in a new algorithm \textsf{KernelGrid}.  We can next develop two improvements to this basic algorithm which make the grid adaptive in resolution, and round the effect of points far from the epicenter; embodied in our algorithm \textsc{KernelFast}, these considerably increase the efficiency of computing the statistic (by $30$x) without significantly decreasing their accuracy (with provable guarantees).  
Moreover, we create a coreset $B_\eps$ of the full data set $B$, independent of the original size $|B|$ that provably bounds the worst case error $\eps$.  

We empirically demonstrate the efficiency, scalability, and accuracy of these new KSSS algorithms.  In particular, we show the KSSS has superior statistical power compared to traditional SSS algorithms, and exceeds the efficiency of even the heavily optimized version of those combinatorial Disk SSS algorithms.



\section{Derivation of the Kernel Spatial Scan Statistic}
\label{sec:model}
In this section we will provide a general definition for a spatial scan statistic, and then extend this to the kernel version. 
It turns out, there are two reasonable variations of such statistics, which we call the continuous and binary settings.  
In each case, we will then define the associated kernelized statistical discrepancy function $\Phi$ under the Bernoulli $\Phi^{\text{Be}}$, Poisson $\Phi^P$, and Gaussian $\Phi^G$ models.    These settings are the same in the Bernoulli model, but different in the other two models.

\Paragraph{General derivation}
A spatial scan statistic considers a spatial data set $B \subset \R^d$, each data point $x \in B$ endowed with a measured value $m(x)$, and a family of measurement regions $\RR$.  Each region $R \in \RR$ specifies the way a data point $x \in B$ is associated with the anomaly (e.g., affected or not affected).  
Then given a statistical discrepancy function $\Phi$ which measures the anomalousness of a region $R$, the statistic is $\max_{R \in \RR} \Phi(R)$.  
To complete the definition, we need to specify $\RR$ and $\Phi$, which it turns out in the way we define $\RR$ are more intertwined that previously realized.  

To define $\Phi$ we assume a statistical model in how the values $m(x)$ are realized, where data points $x$ affected by the anomaly have $m(x)$ generated at rate $p$ and those unaffected generated at rate $q$. 


Then we can define a null hypothesis that a potential anomalous region $R$ has no effect on the rate parameters so $p=q$; and an alternative hypothesis that the region does have an effect and (w.l.o.g.) $p > q$.

For both the null and alternative hypothesis, and a region $R \in \RR$, we can then define a likelihood, denoted $L_0(q)$ and $L(p,q,R)$, respectively.  The spatial scan statistic is then the log-likelihood ratio (LRT)
\[
\Phi(R)  = \log\left( \frac{\max_{p,q} L(p,q,R)}{\max_q L_0(q)} \right) =  \left(\max_{p,q} \log L(p,q,R)\right) - \left(\max_q \log L_0(q)\right).
\]

Now the main distinction with the kernel spatial scan statistic is that $\RR$ is specified with a family of kernels $\K$ so that each $K \in \K$ specifies a \emph{probability} $K(x)$ that $x \in B$ is affected by the anomaly.  This is consistent with the traditional spatial scan statistic (e.g., with $\RR$ as disks $\D$), where this probability was always $0$ or $1$.  Now this probability can be any continuous value. Then we can express the mean rate/intensity $g(x)$ for each of the underlying distributions from which $m(x)$ is generated, as a function of $K(x)$, $p$, and $q$. 
Two natural and distinct models arise for kernel spatial scan statistics. 

\Paragraph{Continuous Setting} 
In the \emph{continuous setting}, which will be our default model, we directly model the mean rate $g(x)$ as a convex combination between $p$ and $q$ as, 
    \[ g(x) = K(x)p + (1 - K(x))q.  \]
Thus each $x$ (with nonzero $K(x)$ value) has a slightly different rate.  
Consider a potentially hazardous chemical leak suspected of causing cancer, this model setting implies that the residents who live closer to the center ($K(x)$ is larger) of potential leak would be affected more (have elevated rate) compared to the residents who live farther away.  
The kernel function $K(x)$ models a decay effect from a center, and smooths out the effect of distance.

\Paragraph{Binary Setting} 
In the second setting, the \emph{binary setting}, the mean rate $g(x)$ is defined 
\[
\breve g(x)= \begin{cases}
 p & \text{ w.p. } K(x)  
 \\ 
 q & \text{ w.p. }(1 - K(x)).
\end{cases}
\] 
To clarify notation, each part of the model associated with this setting (e.g., $\breve g$) will carry a $\breve{}$\,  to distinguish it from the continuous setting.  
In this case, as with the traditional SSS, the rate parameter for each $x$ is either $p$ or $q$, and cannot take any other value.  However, this rate assignment is not deterministic, it is assigned with probability $K(x)$, so points closer to the epicenter (larger $K(x)$) have higher probability of being assigned a rate $p$.  The rate $\breve g(x)$ for each $x$ is a mixture model with known mixture weight determined by $K(x)$.

\Paragraph{The null models}
We show that the choice of model, binary setting, vs continuous setting, does not change the null hypothesis (e.g.  $\ell_0 = \breve \ell_0$).  


\subsection{Bernoulli}
Under the Bernoulli model, the measured value $m(x) \in \{0,1\}$, and these values are generated independently. 
Consider a $1$ value indicating that someone is diagnosed with cancer.  Then an anomalous region may be associated with a leaky chemical plant, where the residents nearby the plant have an elevated rate of cancer $p$, whereas the background population may have lower rate $q$ of cancer.  That is, the cancer occurs through natural mutation at a rate $q$, but if exposed to certain chemicals, there is another mechanism to get cancer that occurs at rate $p-q$ (for a total rate of $q + (p-q) = p$).  
Under the binary model, \emph{any} exposure to the chemical triggers this secondary mechanism, and so the chance of exposure is modeled as proportional to $K(x)$, and rate at $x$ is $\breve g(x)$ is modeled well by the binary setting.  
Alternatively, the rate of the secondary mechanism may increase as the amount of exposure to the chemical increases (those living closer are exposed to more chemicals), with rate $g(x)$ modeled in the continuous setting.  These are both potentially the correct biological model, so we analyze both of them. 

For this model we can define two subsets of $B$ as $M = \{x \in B \mid m(X) =1\}$ and $B \setminus M = \{x \in B \mid m(x) = 0\}$. 

In either setting the null likelihood is defined
\[L_0^\Ber(q) =  \breve L_0^\Ber(q) = \prod_{x \in M} q \prod_{x \in B \setminus M} (1 - q), \] 
then 
\[
  \ell_0^\Ber(q) = \breve \ell_0^\Ber(q) = \sum_{x \in M} \log(q) + \sum_{x \in B \setminus M} \log(1-q),
\]
which is maximized over $q$ at $q = |M|/|B|$ as, 
\[\ell_0^{\Ber^*} = \breve \ell_0^{\Ber^*} = |M| \log\frac{|M|}{|B|} + (|B| - |M|)\log(1 - \frac{|M|}{|B|}). \]

\Paragraph{The continuous setting}
We first deriving the continuous setting $\Phi^\Ber$, starting with the likelihood under the alternative hypothesis. 

This is a product of the rate of measured and baseline. 
\[
L^\Ber(p,q,K) 
=
\prod_{x \in M} g(x) \cdot \prod_{x \in B \setminus M} (1-g(x))
\]
and so 
\begin{align*}
&\ell^\Ber(p,q,K) 
\\&= 
\log(L^\Ber(p,q,K)) 
=
\sum_{x \in M} \log g(x) + \sum_{x \in B \setminus M} \log (1-g(x))
\\ &=
\sum_{x \in M} \log(p K(x) + q (1-K(x)))  + \sum_{x \in B \setminus M} \log(1-p K(x) - q(1-K(x))).
\end{align*}
Unfortunately, we know of no closed form for the maximum of $\ell^\Ber(p,q,K)$ over the choice of $p,q$ and therefore this form cannot be simplified further than
\[
{\Phi^\Ber}^*(K) = \max_{p,q} \ell^\Ber(p,q,K) - {\ell_0^\Ber}^*
\]

\Paragraph{The binary setting}
We continue deriving the binary setting $\breve{\Phi}^\Ber$, starting with 
\[
\breve L^\Ber(p,q,K) 
=
\prod_{x \in B} \left( p^{m(x)}(1-p)^{m(x)}K(x) + q^{m(x)}(1-q)^{m(x)}K(x)\right),
\]
and so 
\begin{align*}
\breve \ell^\Ber(p,q,K) 
&= 
\log(\breve L^\Ber(p,q,K)) 
\\ &=
\sum_{x \in B} \log \left( p^{m(x)}(1-p)^{m(x)}K(x) + q^{m(x)}(1-q)^{m(x)}K(x) \right).
\end{align*}

Similarly as in the continuous setting, there is no closed form of the maximum of $\log(\breve L^\Ber(p,q,K))$ over choices of $p$,$q$, so we write the form below.
\[
\breve{\Phi}^{\Ber^*}(K) = \max_{p,q} \breve \ell^\Ber(p,q,K)  - {\ell_0^\Ber}^* 
\]

\Paragraph{Equivalence}
It turns out, under the Bernoulli model, these two settings have equivalent statistics to optimize.  

\begin{lemma}
The $\breve \ell^\Ber(p,q,K)$ and $\ell^\Ber(p,q,K)$ are exactly same, hence $\breve \Phi^{\Ber^*}$ = ${\Phi^\Ber}^*$ which implies that the Bernoulli model under two settings are equivalent to each other. 
\end{lemma}

\begin{proof}  
We simply expand the binary setting as follows.  
\begin{align*}
& \breve \ell^\textsf{Be}(p,q,K) 
\\ &= 
\sum_{x \in M} \log( p^{m(x)} (1 - p)^{1 - m(x)} K(x) + q^{m(x)} (1 - q)^{1 - m(x)} (1 - K(x))) 
\\ &
 + \sum_{x \in B \setminus M} \log( p^{m(x)} (1 - p)^{1 - m(x)} K(x) + q^{m(x)} (1 - q)^{1 - m(x)} (1 - K(x)))
\\ & =  
\sum_{x \in M} \log( p K(x) + q(1 - K(x))) +  \sum_{x \in B\setminus M} \log(1 - (p-q)K(x) - q) 
\\ & = 
\sum_{x \in M} \log(g(x)) +  \sum_{x \in B\setminus M} \log(1 - g(x)) 
\\ & =  
\ell^\textsf{Be}(p,q,K) 
\end{align*}
Since the $\breve \ell^\Ber(p,q,K) = \ell^\Ber(p,q,K)$ and $\breve \ell_0^{\Ber^*} = {\ell_0^\Ber}^*$, then $\breve \Phi^{\Ber^*}$ = ${\Phi^\Ber}^*$. 
\end{proof}

\subsection{Gaussian} 
The Gaussian model can be used to analyze spatial datasets with continuous values $m(x)$ (e.g., temperature, rainfall, or income), which we assume varies with a normal distribution with a fixed known standard deviation $\sigma$.  Under this model, both the continuous and binary settings are again both well motivated.  

Consider an insect infestation, as it affects agriculture.  Here we assume fields of crops are measured at discrete locations $B$, and each has an observed yield rate $m(x)$, which under the null models varies normally around a value $q$.  In the continuous setting, the yield rate at $x \in B$ is effected proportional to $K(x)$, depending on how close it is to the epicenter.  This may for instance model that fewer insects reach further from the epicenter, and the yield rate is effected relative to the number of insects that reach the field $x$.  
In the binary setting, it may be that if insects reach a field, then they dramatically change the yield rate (e.g., they eat and propagate until almost all of the crops are eaten).  In the latter scenario, the correct model is binary one, with a mixture model of two rates governed by $K(x)$, the closeness to the epicenter.  



In either setting the null likelihood is defined as, 
\[L_0^\Gau(q) =  \breve L_0^\Gau(q) = \prod_{x \in B} \exp(-\frac{(m(x) -q)^2}{2 \sigma^2}), \] 
then
\[\ell_0^\Gau(q) =  \breve \ell_0^\Gau(q) = \sum_{x \in B} -\frac{(m(x) -q)^2}{2 \sigma^2}, \] 
which is maximized over $q$ at $q = \frac{\sum_{x \in B} m(x)}{|B|} = \hat m$ as, 
\[\ell_0^{\Gau^*} = \breve \ell_0^{\Gau^*} = \frac{-1}{2 \sigma^2} \sum_{x \in B}(\hat m - m(x))^2. \]

\Paragraph{The continuous setting} 
We first derive the continuous setting $\Phi^\Gau$ and terms free of $p$ and $q$ would be treated as constant then ignored, starting with, 
\[
L^\Gau(p,q,K) 
=
\prod_{x \in B} \exp(-\frac{(m(x) - g(x))^2}{2\sigma^2})
\]
and so 
\begin{align*} 
\ell^\Gau(p,q,K) 
&= 
\log(L^\Gau(p,q,K))  = \sum_{x \in B} -\frac{(m(x) - g(x))^2}{2\sigma^2} \\ &= 
\sum_{x \in B} -\frac{(m(x) - pk(x) - q(1 - K(x))^2}{2\sigma^2}.
\end{align*}

Fortunately, there is a closed form for the maximum of $\ell^\Gau(p,q,K)$ over the choice of $p,q$ by setting the $\frac{\dir \ell^\Gau(p,q,K)}{\dir p} = 0$ and $\frac{\dir \ell^\Gau(p,q,K)}{\dir q} = 0$. Hence we come up with the closed form solution of Gaussian Kernel statistical discrepancy shown by the theorem below. 

\begin{theorem}
\label{thm:GKSSS}
Gaussian kernel statistical discrepancy function is
\[
\Phi^{\Gau^*}(K) = \sum_{x \in B} (m(x) - \hat p K(x) - \hat q (1-K(x)))^2 - \sum_{x\in B}(\bar m - m(x))^2
\]
where 
$\bar m = \frac{1}{|B|} \sum_{x \in B} m(x) $, 
$\hat p = \frac{K_{\pm}K_{-m} - K_m K_{-2}}{K_{\pm}^2 - K_2 K_{-2}}$,  
$\hat q = \frac{K_m K_{\pm} - K_2 K_{-m}}{K_{\pm}^2 - K_2 K_{-2}}$,  using the following terms 
$K_m = \sum_{x \in B} K(x) m_{x}$, \;
$K_2 = \sum_{x \in B} K(x)^2$, \;\;
$K_{\pm} = \sum_{x \in B} K(x) (1 - K(x))$, 
$K_{-m} = \sum_{x \in B} m(x) (1 - K(x))$, and 
$K_{-2} = \sum_{x \in B} (1 - K(x))^{2}$.
\end{theorem}
\begin{proof}

For the alternative hypothesis, the log-likelihood is
\[
\ell^\Gau(p,q,K) = \frac{-1}{2\sigma^2} \sum_{x \in B} (m(x) - g(x))^2
\]
The optimal values of $p,q$ minimize
\[
-\ell^\Gau(p,q,K)  
= 
\frac{1}{2\sigma^2} \sum_{x \in B} (m(x) - pK(x) - q(1-K(x)))^2.  
\]

By setting the partial derivatives wrt $p$ and $q$ of $-\ell^\Gau(p,q,L)$ equal to $0$, we have, 
\[
\frac{\dir \ell^\Gau(p,q,K)}{\dir p} 
= \sum_{x \in B} K(x)[m(x) - K(x)p - (1 - K(x)q] = 0,
\]
and 
\[ 
\frac{\dir \ell^\Gau(p,q,K)}{\dir q} 
= \sum_{x \in B} (1 - K(x))[m(x) - K(x)p - (1 - K(x)q] = 0,
\]
and these two can be further simplified to,
\begin{align*}
\frac{\dir \ell^\Gau(p,q,K)}{\dir p} 
&= 
\sum_{x \in B} K(x)m(x) - p\sum_{x \in B}K(x)^2 - q\sum_{x \in B}K(x)(1 - K(x)) 
\\&= 
0,
\end{align*}
and
\begin{align*}
&\frac{\dir \ell^\Gau(p,q,K)}{\dir q} 
\\&= \sum_{x \in B} (1 - K(x))m(x) - p\sum_{x \in B}(1 - K(x))K(x) - q\sum_{x \in B}(1 - K(x))^2 \\&= 0.
\end{align*}
We replace these terms by notations defined in the theorem, 
\[
   K_m - pK_{2} - qK_{\pm} = 0,
\]
and 
\[
  K_{-m} - pK_{\pm} - qK_{-2} = 0.
\]
Then we can solve the optimum value of $p$ and $q$ as, 
$p = \frac{K_{\pm}K_{-m} - K_m K_{-2}}{K_{\pm}^2 - K_2 K_{-2}} = \hat p$
and
$q =  \frac{K_m K_{\pm} - K_2 K_{-m}}{K_{\pm}^2 - K_2 K_{-2}} = \hat q$.

Hence we have the closed form
\[
 \ell^{\Gau^*} = \max_{p,q} \ell^\Gau(p,q,K) = \frac{-1}{2\sigma^2}\sum_{x \in B} (m(x) - \hat p K(x) - \hat q (1-K(x)))^2.
\]
\end{proof}

\Paragraph{The binary Setting}
Now we derive the binary setting $\breve{\Phi}^\Ber$, starting 
\[
 \breve L^\Gau(p,q,K) = \prod_{x \in B}  \exp(-\frac{(m(x) - p)^2}{2 \sigma^2})K(x) + \exp(-\frac{(m(x) - q)^2}{2 \sigma^2})(1 - K(x))
\]
and so, 
\begin{align*}
 &\breve \ell^\Gau(p,q,K) \\&= \sum_{x \in B} \log(\exp(-\frac{(m(x) - p)^2}{2 \sigma^2})K(x) + \exp(-\frac{(m(x) - q)^2}{2 \sigma^2})(1 - K(x))).  
\end{align*}

Different from the continuous setting we know of no closed form for the maximum of $\breve \ell^\Gau(p,q,K)$ over $p$ and $q$. Hence,
\[
   \breve{\Phi}^{\Gau^*}(K) = \max_{p,q} \breve \ell^\Gau(p,q,K)  - {\ell_0^\Gau}^*.  
\]

\Paragraph{Equivalence}
Different from the Bernoulli models, the Gaussian models under the binary setting and continuous setting are not equivalent to each other. Under the continuous setting,  
\[ m(x) \sim \mathcal{N}(g(x),\sigma), \]
however, under the binary setting, each data point follow a two components Gaussian mixture model where the mixture weight is given by $K(x)$ and $1 - K(x)$, and so it is not a Gaussian distribution anymore. 





\subsection{Poisson}

In the Poisson model the measured value $m(x)$ is discrete and non-negative, but it can now take any positive integer value with $m(x) \in \mathbb{Z}_+$.  This can for instance model the number of check-ins or comments $m(x)$ posted at each geo-located business $x \in B$ (this can be a proxy for instance for the number of customers).  An event, e.g., a festival, protest, or other large impromptu gathering could be  modeled spatially by a kernel $K$, and it affects the rates at each $x$ in the two different settings.  

In the continuous setting, the closer a distance a restaurant is from the center of the event (modeled by $K(x)$) the more the usual number of check-ins (modeled by $q$) will trend towards $p$.
On the hand, in the binary setting, only certain businesses are affected (e.g., a coffee shop, but not a fancy dinner location), but if it is affected, its rate is elevated all the way from $q$ to $p$.  Perhaps advertising at a festival encouraged people to patronize certain restaurants, or a protest encouraged them to give bad reviewers to certain nearby restaurants -- but not others.  Hence, these two settings relate to two different ways an event could affect Poisson check-in or comment rates.




In either setting, the null likelihood is defined as,
\[L_0^\Poi(q) =  \breve L_0^\Poi(q) = \prod_{x \in B} \frac{e^{-q}q^{m(x)}}{m(x)!}, \] 
then
\begin{align*}
\ell_0^\Poi(q) =  \breve \ell_0^\Poi(q) &= \sum_{x \in B} \log(\frac{e^{-q}q^{m(x)}}{m(x)!}) 
\\& = \sum_{x \in B} -q + m(x)\log(q) - \log(m(x)!),
\end{align*}
which is maximized over $q$ at $q = \frac{\sum_{x \in B} m(x)}{|B|} = \hat m$ as, 
\begin{align*}
\ell_0^{\Poi^*} 
= 
\breve \ell_0^{\Poi^*} 
& = 
\sum_{x \in B} -\hat m + m(x) \log(\hat m) - \log(m(x)!)
\\ &= 
- |B| \hat m - \sum_{x \in B} \log (m(x)!) - m(x) \log (\hat m).
\end{align*}

\Paragraph{The continuous setting}
We first derive the continuous setting $\Phi^\Poi$, starting with, 
\begin{align*}
    L^\Poi(p,q,K) 
    &= \prod_{x \in B} \frac{e^{-g(x)} g(x)^{m(x)}}{m(x)!}
    \\&= \prod_{x \in B} \frac{e^{-pK(x) -q(1- K(x))} (pK(x) + q(1-K(x))^{m(x)}}{m(x)!},
\end{align*}
so 
\begin{align*}
    &\ell^\Poi(p,q,K) 
    \\&= \sum_{x \in B} -g(x) + m(x)\log(g(x)) - \log(m(x)!)
    \\&= \sum_{x \in B} -(pK(x) + q(1-K(x)) + m(x)\log(pK(x) + q(1-K(x)) - \log(m(x)!))
\end{align*}
There is no closed form for the maximum of $\ell^\Poi(p,q,K)$ over the choice of $p$ and $q$ and hence 
\[
    \Phi^{\Poi^*}(K) = \max_{p,q}\ell^\Poi(p,q,K) - \ell_0^{\Poi^*}(q)
\]

\Paragraph{The binary setting}
We continue deriving the binary setting $\breve \Phi^{\Poi^*}$, starting with
\[ 
   \breve L^\Poi(p,q,K) = \prod_{x \in B} \frac{e^{-q} q^{m(x)}}{m(x)!} (1 - K(x)) +  \frac{e^{-p} p^{m(x)}}{m(x)!} K(x),
\]
so
\[ 
   \breve \ell^\Poi(p,q,k) = \sum_{x \in B} \log\left(\frac{e^{-q} q^{m(x)}}{m(x)!} (1 - K(x)) +  \frac{e^{-p} p^{m(x)}}{m(x)!} K(x) \right).  
\]

Same as the continuous setting, there is no closed form of  $\breve \ell^\Poi(p,q,k)$, hence, 

\[ \breve \Phi^{\Poi^*}(K) = \max_{p,q} \breve \ell^\Poi(p,q,k) - \breve \ell_0^{\Poi^*}(q) \]

\Paragraph{Equivalence}
In the Poisson model the binary setting and the continuous setting are not equivalent to each other. Under the continuous setting, 
\[ 
m(x) \sim \mathsf{Poi}(g(x)), 
\]
however under the binary setting $m(x)$ follows a mixture Poisson model which is not a Poisson distribution anymore and the mixture weight is given by $K(x)$ and $1 - K(x)$.

\section{Computing the Approximate KSSS}
\label{sec:algorithm}
The traditional SSS can combinatorially search over all disks~\cite{SSSS,MP18a,Kul97} to solve for or approximate $\max_{D \in \D}\Phi(D)$, evaluating $\Phi(D)$ exactly.  Our new KSSS algorithms will instead search over a grid of possible centers $c$, and approximate $\Phi(K_c)$ with gradient descent, yet will achieve the same sort of strong error guarantees as the combinatorial versions.  Improvements will allow for adaptive gridding, pruning far points, and sampling.  


\subsection{Approximating $\Phi$ with GD}
\label{sec:approx-log-lambda}
We cannot directly calculate $\Phi(K) = \max_{p,q} \Phi_{p,q}(K)$, since it does not have a closed form.  Instead we run gradient descent $\mathsf{GradDesc}$ over $-\Phi_{p,q}(K_c)$ on $p,q$ for a fixed $c$.  Since we have shown $\Phi_{p,q}(K)$ is convex over $p,q$ this will converge, and since Lemma \ref{lem:lipshitz-pq} bounds its Lipschitz constant at $2 |B|$ it will converge quickly.  
In particular, from starting points $p_0,q_0$, after $s$ steps we can bound
\[
|\Phi_{p^*, q^*}(K) - \Phi_{p_s, q_s}(K)| \le \frac{|B| \|(p^*,q^*) - (p_0,q_0)\| }{s},
\]
for the found rate parameters $p_s,q_s$.  Since $0 \leq p,q \leq 1$, then after $s_\eps = |B|/\eps$ steps we are guaranteed to have 
$|\Phi_{p^*, q^*}(K) - \Phi_{p_{s_\eps}, q_{s_\eps}}(K)| \leq \eps$.  
We always initiate this procedure on $\Phi_{p,q}(K_c)$ with the $\hat p, \hat q$ found on a nearby $K_{c'}$, and as a result found that running for $s = 3$ of $4$ steps is sufficient.  
Each step of gradient descent takes $O(|B|)$ to compute the gradient.  
\subsection{Gridding and Pruning}
\label{sec:algorithm-description}

Computing $\Phi(K)$ on every center in $\mathbb{R}^2$ is impossible, but Lemma \ref{lem:spatial-lipshitz} shows that $\Phi(K_{\hat c})$ with $\hat c$ close to the true maximum $c^*$, then $\Phi(K_{\hat c})$ will  approximate the true maximum.  
From Lemma \ref{lem:spatial-lipshitz} we have $|\Phi(K_{\hat{c}}) - \Phi(K_{c^*})| \le \frac{1}{r} \sqrt{\frac{8}{e}} \|\hat{c} - c^*\|$.  To get a bound of $|\Phi(K_{\hat{c}}) - \Phi(K_{c^*})| \le \eps$ we need that $\|\hat{c} - c^*\|\frac{1}{r}\sqrt{\frac{8}{e}} \le \eps$ or rearanging that $\|\hat{c} - c^*\| \le \eps r \sqrt{\frac{e}{8}}$.  Placing center points on a regular grid with sidelength $\tau_\eps = \eps r \sqrt{\frac{e}{4}}$ will ensure 
the a center point will be close enough to the true maximum. 

Assume that $B \subset \Omega_\Lambda \subset \R^2$, which w.l.o.g. we assume $\Omega_\Lambda \in [0,L] \times [0,L]$, where $\Lambda = L/r$ is a unitless resolution parameter which represents the ratio of the domain size to the scale of the anomalies.  
Next define a regular orthogonal grid $G_\eps$ over $\Omega_\Lambda$ at resolution $\tau_\eps$.  This contains $|G_\eps \cap \Omega_\Lambda| = \frac{\Lambda^2}{\eps^2}\frac{4}{e}$ points.  
We compute the scan statistic $\Phi(K_c)$ choose as $c$ each point in $G_\eps \cap \Omega_\Lambda$.  
This algorithm, denoted \textsc{KernelGrid} and shown in Algorithm \ref{alg:D2}, can be seen to run in $O(|G_\eps \cap \Omega_\Lambda| \cdot |B| t_{g} ) = O(\frac{\Lambda^2}{\eps^2} |B| s_\eps)$ time, using $s_\eps$ iterations of gradient decent (in practice $s_\eps = 4$). 

\begin{lemma}
\label{lem:kernel-grid}
$\textsc{\em KernelGrid}(B,\eps,\Omega_\Lambda)$ returns $\Phi(K_{\hat c})$ for a center $\hat c$ 
so
$|\max_{K_c \in \K_r} \Phi(K_c) - \Phi(K_{\hat c})| \le \eps$ 
in time $O(\frac{\Lambda^2}{\eps^2} |B| s_\eps)$, which in the worst case is $O(\frac{\Lambda^2}{\eps^3} |B|^2)$.
\end{lemma}

 \begin{algorithm}
	\caption{\textsc{KernelGrid}$(B, \eps, \Omega_\Lambda)$}
	\label{alg:D2}
	\begin{algorithmic}
	    \STATE Initialize $\Phi = 0$; define $G_{\eps,\Lambda} = G_\eps \cap \Omega_\Lambda$
		\FOR{$c \in G_{\eps,\Lambda}$}
		  \STATE $\Phi_c = \mathsf{GradDesc}(\Phi_{p, q}(K_c))$ over $p,q$ on $B$
		   \STATE \textbf{if } ($\Phi_c > \Phi$) \textbf{ then } $\Phi = \Phi_c$
		\ENDFOR
        \STATE return $\Phi$
		\end{algorithmic}
\end{algorithm}

\Paragraph{Adaptive Gridding}
We next adjust the grid resolution based on the density of $B$.  
We partition the $\Omega_\Lambda$ domain with a coarse grid $H_\eps$ with side length $2r_{max}$ (from Lemma \ref{lem:truncate}). 
For a cell $\gamma \in H_\eps$ in this grid, let $S_\gamma$ denote the $6 r_{\textrm{max}} \times 6 r_{\textrm{max}}$ region which expands the grid cell $\gamma$ the length of one grid cell in either direction.  For any center $c \in \gamma$, all points $x \in B$ which are within a distance of $2r_{\textrm{max}}$ from $c$ must be within $S_\gamma$.  Hence, by Lemma \ref{lem:truncate} we can evaluate $\Phi(K'_c)$ for any $c \in \gamma$ only inspecting $S \cap B$.  

Moreover, by the local density argument in Lemma \ref{lem:spatial-lipshitz-adaptive}, we can describe a new grid $G'_{\eps,\gamma}$ inside of each $\gamma \in H_\eps$ with center separation $\beta$ only depending on the local number of points $|S_\gamma \cap B|$.  In particular we have for $c,c' \in \gamma$ with $\|c-c'\|=\beta$
\[
|\Phi(K'_c) - \Phi(K'_{c'})|
\le  
\beta \frac{|S \cap B|}{|B|}\frac{2 r_\text{max}}{r^2} + \eps
\]
To guarantee that all $c \in \gamma$ have another center $c' \in G'_{\eps,\gamma}$ so that 
$|\Phi(K_c) - \Phi(K'_{c'})| \leq 2\eps$
we set the side length in $G'_{\eps,\gamma}$ as 
\[
\beta_\gamma = \eps \frac{|B|}{|B \cap S_\gamma|} \frac{r^2}{2r_\textrm{max}},
\]
so the number of grid points in $G'_{\eps,\gamma}$ is
\[
|G'_{\eps,\gamma}| 
= 
\frac{4 r_\textrm{max}^2}{\beta_\gamma^2} 
= 
\frac{r^4_\text{max}}{r^4} 
\frac{16}{\eps^2} \frac{|B \cap S_\gamma|^2}{|B|^2}. 
\]

Now define $G'_\eps = \bigcup_{\gamma \in H_\eps} G'_{\eps,\gamma}$ as the union of these adaptively defined subgrids over each coarse grid cell.  Its total size is
\begin{align*}
|G'_\eps| 
&= 
\sum_{\gamma \in H_\eps} |G_{\eps,\gamma}| 
=
\sum_{\gamma \in H_\eps} \frac{r^4_\text{max}}{r^4} 
\frac{16}{\eps^2} \frac{|B \cap S_\gamma|^2}{|B|^2}.  
\\ &=
\frac{r^4_\text{max}}{r^4} 
\frac{16}{\eps^2}\cdot \frac{1}{|B|^2}
\sum_{\gamma \in H_\eps}  |B \cap S_\gamma|^2.  
\\ &=
\log^2(|B|/\eps)
\frac{16}{\eps^2}
\sum_{\gamma \in H_\eps}  \cdot \frac{1}{|B|^2}
\sum_{\gamma \in H_\eps}  |B \cap S_\gamma|^2.  
\\ &\leq
\log^2(|B|/\eps)
\frac{1296}{\eps^2},
\end{align*}
where the last inequality follows from Cauchy-Schwarz, and that each point $x \in B$ appears in $9$ cells $S_\gamma$.   This replaces dependence on the domain size $\Lambda^2$ in the size of the grid $G_\eps$, with a mere logarithmic $\log^2(|B|/\eps)$ dependence on $|B|$ in $G'_\eps$.  We did not minimize constants, and in practice we use significantly smaller constants.  

Moreover, since this holds for each $c \in \gamma$, and the same gridding mechanism is applied for each $\gamma \in H_\eps$, this holds for all $c \in \Omega_\Lambda$.  
We call the algorithm that extends Algorithm \ref{alg:D2} to use this grid $G'_\eps$ in place of $G_\eps$ \textsc{KernelAdaptive}.

\begin{lemma}
\label{lem:kernel-adaptive}
$\textsc{\em KernelAdaptive}(B,\eps,\Omega_\Lambda)$ returns $\Phi(K_{\hat c})$ for a center $\hat c$ 
so
$|\max_{K_c \in \K_r} \Phi(K_c) - \Phi(K_{\hat c})| \le \eps$ 
in time $O(\frac{1}{\eps^2} \log^2 \frac{|B|}{\eps} |B| s_\eps)$, which in the worst case is $O(\frac{1}{\eps^3} |B|^2 \log^2 \frac{|B|}{\eps})$.
\end{lemma}

\Paragraph{Pruning} 
For both gridding methods the runtime is roughly the number of centers times the time to compute the gradient $O(|B|)$.   But via Lemma \ref{lem:truncate} we can ignore the contribution of far away points, and thus only need those in the gradient computation.  

However, this provides no worst-case asymptotic improvements in runtime for \textsc{KernelGrid}, or \textsc{KernelAdaptive} since all of $B$ may reside in a $r_\text{max} \times r_\text{max}$ cell.  
But in the practical setting we consider, this does provide a significant speedup as the data is usually spread over a large domain that is many times the size of $r_\text{max}$.

We will define two new methods \textsc{KernelPrune} and \textsc{KernelFast} (shown in Algorithm \ref{alg:Kfast}) where the former extends \textsc{KernelGrid} method, and latter extends \textsc{KernelAdaptive} with pruning.  
In particular, we note that bounds in Lemma \ref{lem:kernel-adaptive} hold for \textsc{KernelFast}.

\begin{algorithm}
\caption{\textsc{KernelFast}$(B, \eps, \Omega_\Lambda)$}
\label{alg:Kfast}
\begin{algorithmic}
  \STATE Initialize $\Phi = 0$; define grid $H_{\eps}$ over $\Omega_\Lambda$
    \FOR{$\gamma \in H_\eps$}
        \STATE Define $G'_{\eps, \gamma}$ adaptively on $\gamma$ and $S_\gamma \cap B$
    	\FOR{$c \in G'_{\eps,\Lambda}$}
    	  \STATE $\Phi_c = \mathsf{GradDesc}(\Phi_{p, q}(K_c))$ over $p,q$ on \emph{pruned set} $B \cap S_\gamma$
		  \STATE \textbf{if } ($\Phi_c > \Phi$) \textbf{ then } $\Phi = \Phi_c$
		\ENDFOR
	\ENDFOR
  \STATE return $\Phi$
\end{algorithmic}
\end{algorithm}


\vspace{-.1in}
\subsection{Sampling}
\label{sec:alg-sample}

We can dramatically improve runtimes on large data sets by sampling a coreset $B_\eps$ iid from $B$, according to Lemma \ref{lem:sample-bound}.  With probability $1-\delta$ we need $|B_\eps| = O(\frac{1}{\eps^2} \log^2 \frac{1}{\eps} \log \frac{\kappa}{\delta})$ samples, where $\kappa$ is the number of center evaluations, and can be set to the grid size.  
For \textsc{KernelGrid} $\kappa = |G_\eps| = O(\frac{\Lambda^2}{\eps^2})$ and in \textsc{AdaptiveGrid} $\kappa = |G'_\eps| = O(\frac{1}{\eps^2} \log^2\frac{|B_\eps|}{\eps}) = (\frac{1}{\eps^2} \log^2\frac{1}{\eps})$.
We restate the runtime bounds with sampling to show they are independent of $|B|$.


\begin{lemma}
\label{lem:kernel-gridding-sample}
$\textsc{\em KernelGrid}(B_\eps,\eps,\Omega_\Lambda)$ $\&$
$\textsc{\em KernelPrune}(B_\eps,\eps,\Omega_\Lambda)$
with sample size $|B_\eps| = O(\frac{1}{\eps^2} \log \frac{\Lambda}{\eps \delta})$ 
returns $\Phi(K_{\hat c})$ for a center $\hat c$ 
so
$|\max_{K_c \in \K_r} \Phi(K_c) - \Phi(K_{\hat c})| \le \eps$ 
in time 
$O(\frac{s_\eps}{\eps^4} \log \frac{\Lambda}{\eps \delta})$, 
with probability $1 - \delta$.  
In the worst case the runtime is 
$O(\frac{\Lambda^2}{\eps^7} \log^2 \frac{\Lambda}{\eps \delta})$.  
\end{lemma}

\begin{lemma}
\label{lem:adaptive-gridding-sample}
$\textsc{\em KernelAdaptive}(B_\eps,\eps,\Omega_\Lambda)$ $\&$ 
$\textsc{\em KernelFast}(B_\eps,\eps,\Omega_\Lambda)$
with sample size $|B_\eps| = O(\frac{1}{\eps^2} \log \frac{1}{\eps \delta})$ 
returns $\Phi(K_{\hat c})$ for a center $\hat c$ 
so
$|\max_{K_c \in \K_r} \Phi(K_c) - \Phi(K_{\hat c})| \le \eps$ 
in time 
$O(\frac{s_\eps}{\eps^4} \log^3 \frac{1}{\eps \delta})$, 
with probability $1 - \delta$.  
In the worst case the runtime is 
$O(\frac{1}{\eps^7} \log^4 \frac{1}{\eps \delta})$.  
\end{lemma}

\vspace{-.1in}
\subsection{Multiple Bandwidths}
\label{sec:alg-bandwidth}


We next show a sequence of bandwidths, such that one of them is close to the $r$ used in any $K \in \K$ (assuming some reasonable but large range on the values $r$), and then take the maximum over all of these experiments. 
If the sequence of bandwidths $r_{0} \ldots r_{s}$ is such that  $r_{i} - r_{i - 1} \le \frac{e r_{i}\eps}{4}$ then  $|\Phi(K_{r_{i}}) - \Phi(K_{r_{i}})| \le \eps$.
\begin{lemma}
A geometrically increasing sequence of bandwidths $R_\text{min} = r_{0}, \ldots, R_\text{max} = r_s$ with $s = O(\frac{1}{\eps}\log \frac{R_\text{max}}{R_\text{min}})$ is sufficient so $|\max_i \Phi(K_{r_i}) - \Phi(K_r)| \le \eps$ for any bandwidth $r \in  [R_\text{min}, R_\text{max}]$.
\end{lemma}
\begin{proof}
To guarantee a $\eps$ error on the bandwidth $r$ there must be a nearby bandwidth $r_i$. Therefore if $|\Phi(K_{r_{i}}) - \Phi(K_{r_{i + 1}})| \le \eps$ then $|\Phi(K_{r_{i}}) - \Phi(K_{r})| \le \eps$ if $r \in [r_i, r_{i + 1}]$. 

From Lemma \ref{lem:bandwidth} we can use the Lipshitz bound at $r_i$ to note that
$|\Phi(K_{r_{i}}) - \Phi(K_{r_{i + 1}})| \le \frac{4}{e r_i}(r_{i + 1} - r_i)$. Setting this less than $\eps$ we can rearrange to get that  $r_{i + 1} \le (\frac{e}{4}\eps + 1) r_i$. 
That is 
$r_{0} (\frac{\eps e}{4} + 1)^{s} \ge r_{s}$,
which can be rearranged to get $s$
$s = \frac{\log(\frac{R_\text{max}}{R_\text{min}})}{\log(\frac{\eps e}{4} + 1)}$.
Since $\log(x + 1) \ge \frac{x}{2}$ when $x$ is in $(0,1)$, we have  
$s \le \frac{8}{\eps e}\log \frac{R_\text{max}}{R_\text{min}}$.
\end{proof}

Running our KSSS over a large sequence of bandwidths is simple and merely increases the runtime by a $O(\frac{1}{\eps}\log \frac{r_s}{r_0})$ factor.  Our experiments in Section \ref{sec:bandwidth} suggest that choosing $4$ to $6$ bandwidths should be sufficient (e.g., for scales $R_\text{max}/R_\text{min} = 1{,}000$).

\begin{figure*}[h!]
\includegraphics[width=.3\linewidth]{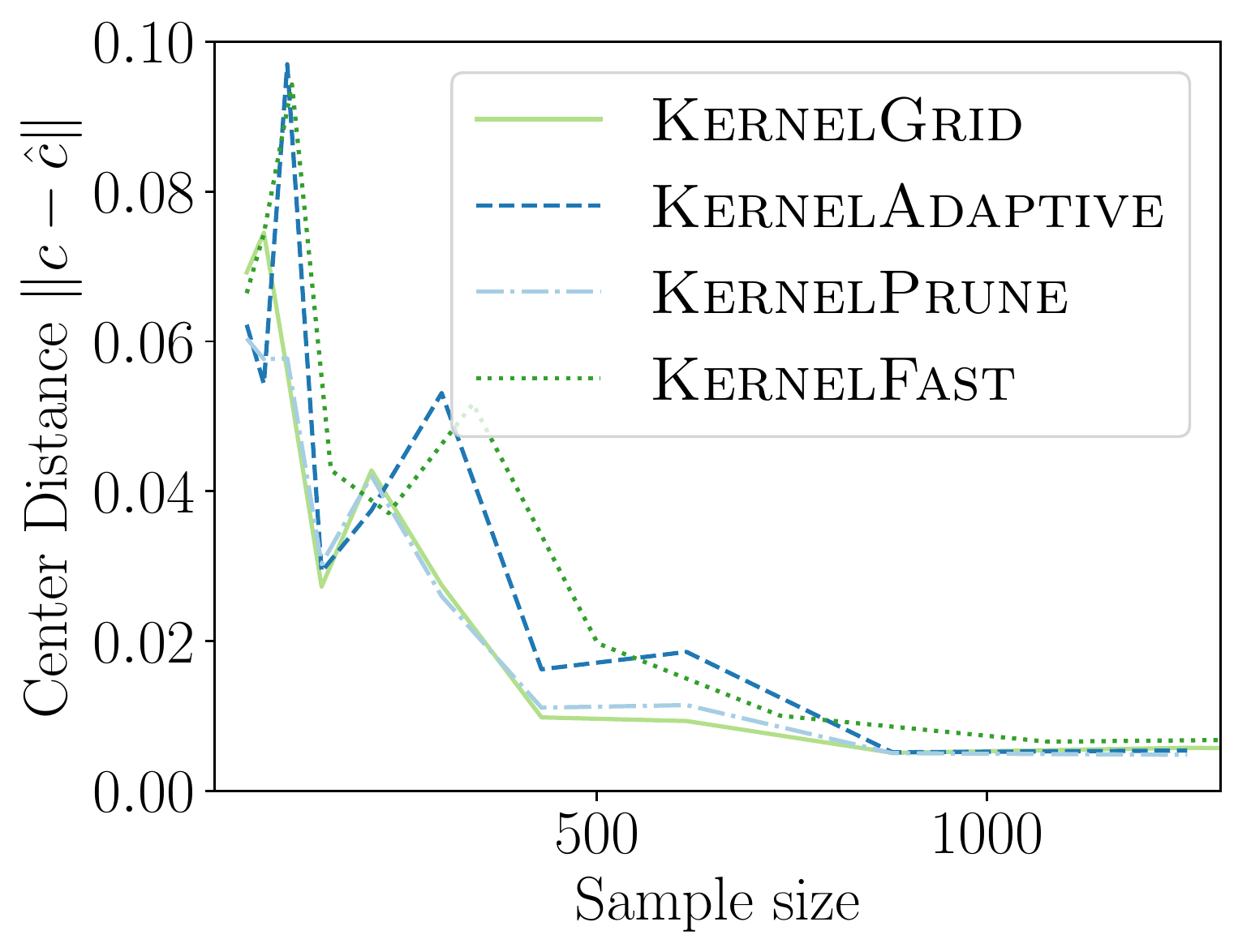}
\includegraphics[width=.3\linewidth]{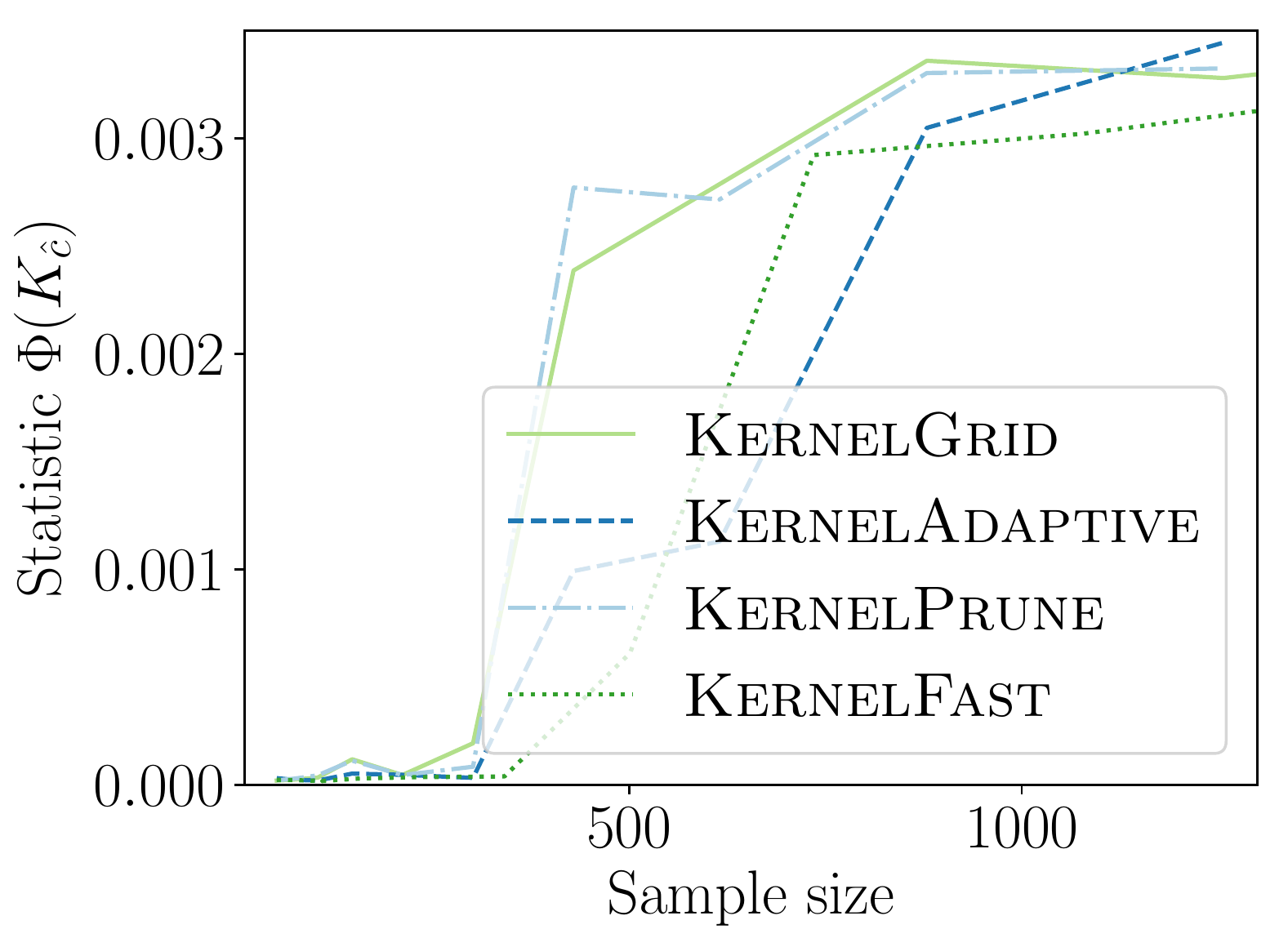}
\includegraphics[width=.3\linewidth]{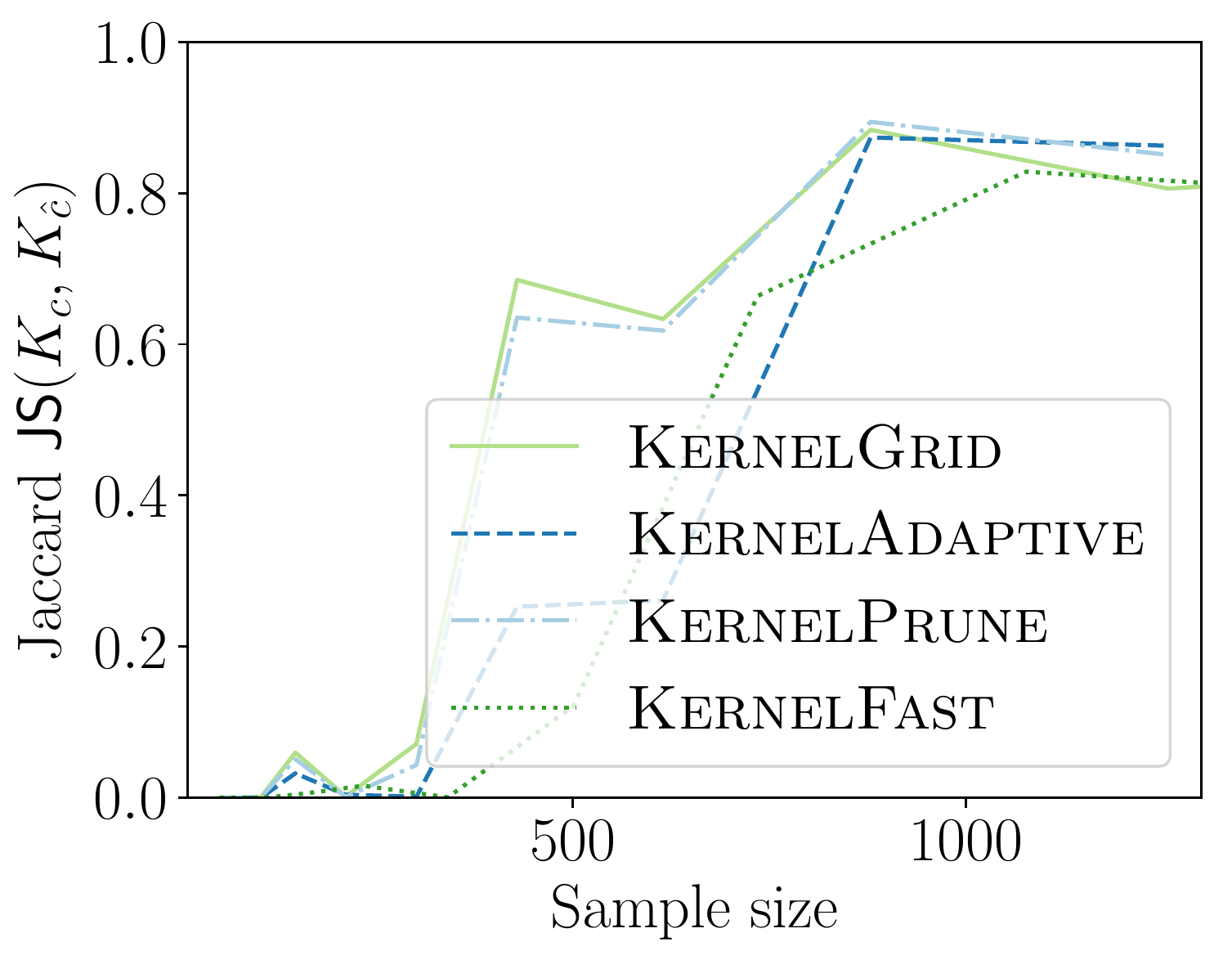}

\vspace{-5mm}
\caption{\label{fig:measure1-kernel}
New KSSS algorithms with statistican power compared with increased sample size using $\|c-\hat c\|$, $\Phi(K_{\hat c})$, and $\mathsf{JS}(K_c, K_{\hat c})$. }
\end{figure*}


\section{Experiments}
\label{sec:experiments}

We compare our new KSSS algorithms to the state-of-the-art methods in terms of empirical efficiency, statistical power, and sample complexity, on large spatial data sets with planted anomalies.

\Paragraph{Data sets} 
We run experiments on two large spatial data sets recording incidents of crime, these are used to represent the baseline data $B$.  
The first contains geo-locations of all crimes in Philadelphia from 2006-2015, and has a total size of $|B| = 687{,}636$; a subsample is shown in Figure \ref{fig:Bernoulli-model}.  
The second is the well-known Chicago Crime Dataset from 2001-2017, and has a total size of $|B| = 6{,}886{,}676$; which is $10x$ the size of the Philadelphia set.

In modeling crime hot spots, these may often be associated with an individual or group of individuals who live at a fixed location.  Then the crimes they may commit would often be centered at that location, and be more likely to happen nearby, and less likely further away.  A natural way to model this decaying crime likelihood would be with a Gaussian kernel --- as opposed to a uniform probability within a fixed radius, and no increased probability outside that zone.  Hence, our KSSS is a good model to potentially detect such spatial anomalies.  





\Paragraph{Planting anomalous regions}
To conduct controlled experiments, we use a spatial data sets $B$ above, but choose the $m$ values in a synthetic way.  In particular, we \emph{plant} anomalous regions $K_c \in \K_r$, and then each data point $x \in B$ is assigned to a group $P$ (with probability $K(x)$) or $Q$ (otherwise).  Those $x \in P$ will be assigned $m(x)$ through a Bernoulli process at rate $p$, that is $m(x) =1$ with probability $p$ and $0$ otherwise; those $x \in Q$ are assigned $m(x)$ at
rate $q$.  Given a region $K_c$, this could model a pattern crimes (those with $m(x) = 1$ may be all vehicle theft or have suspect matching a description), where $c$ may represent the epicenter of the targeted pattern of crime.  We use $p=0.8$ and $q=0.5$.  

We repeat this experiment with $20$ planted regions and plot the median on the Phileadelphia data set to compare our new algorithms and to compare sample complexity properties against existing algorithms.  We use $3$ planted regions on the Chicago data set to compare scalability (these take considerably longer to run).  We attempt to fix the size $P$ so $|P| = f |B|$, by adjusting the fixed and known bandwidth parameter $r$ on each planted region.  We set $f=0.03$ for Philadelphia, and $f=0.01$ for Chicago, so the region contains a fairly small region with about $3\%$ or $1\%$ of the data.


\Paragraph{Evaluating the models}
A statistical power test, plants an anomalous region (for instance as described above), and then determines how often an algorithm can recover that region; it measures recall.  
However, all considered algorithms typically do not recover the exact same region as the one planted, so we measure how close to the planted region $K_c$ the recovered one $K_{\hat c}$ is.  To do so we measure:
\begin{itemize}
    \item distance been found centers $\|c-\hat c\|$, smaller is better.
    \item $\Phi(K_{\hat c})$, the larger the better; it may be larger than $\Phi(K_c)$
    \item the extended Jaccard similarity $\mathsf{JS}(K_c, K_{\hat c})$ defined
    \[
    \mathsf{JS}(K,\hat K) = \frac{\langle K(x), \hat K(x) \rangle_B}{\langle K(x), K(x) \rangle_B + \langle \hat K(x), \hat K(x) \rangle_B  - \langle K(x), \hat K(x) \rangle_B}
    \]
    where $\langle K(x), \hat K(x') \rangle_B = \sum_{x \in B} K(x) \hat K(x)$; larger is better. 
\end{itemize}
We plot medians over $20$ trials; the targeted and hence measured values have variance because planted regions may not be the optimal region, since the $m(x)$ values are generated under a random process. 
When we cannot control the $x$-value (when using time) we plot a kernel smoothing over different parameters on $3$ trials.

\begin{figure}
\vspace{-2mm}
\includegraphics[width=.8\linewidth]{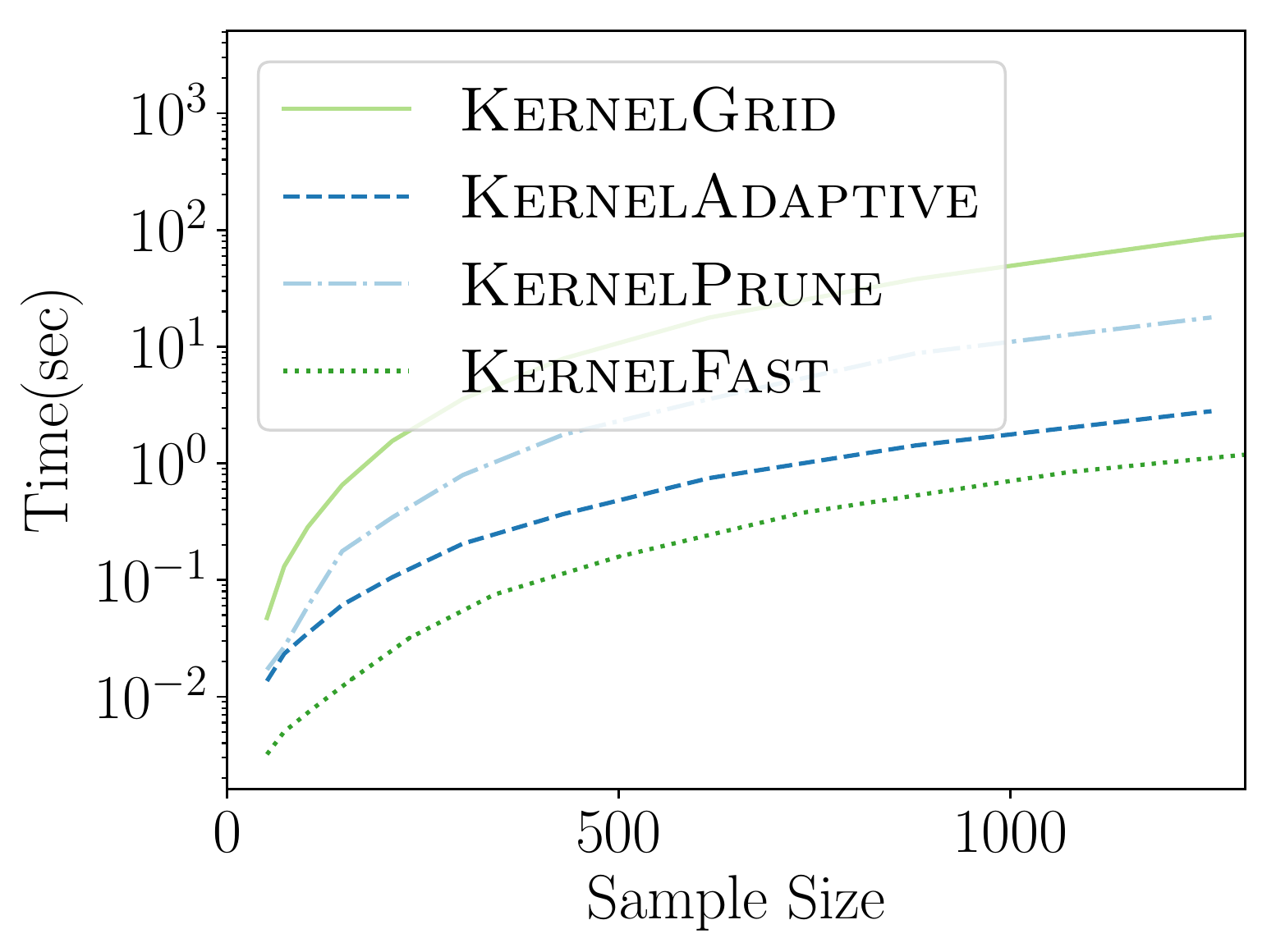}

\vspace{-5mm}
\caption{\label{fig:time-kernel}
Runtime of new KSSS algorithms in sample size.}
\end{figure}

\begin{figure*}
\vspace{-2mm}
\includegraphics[width=.3\linewidth]{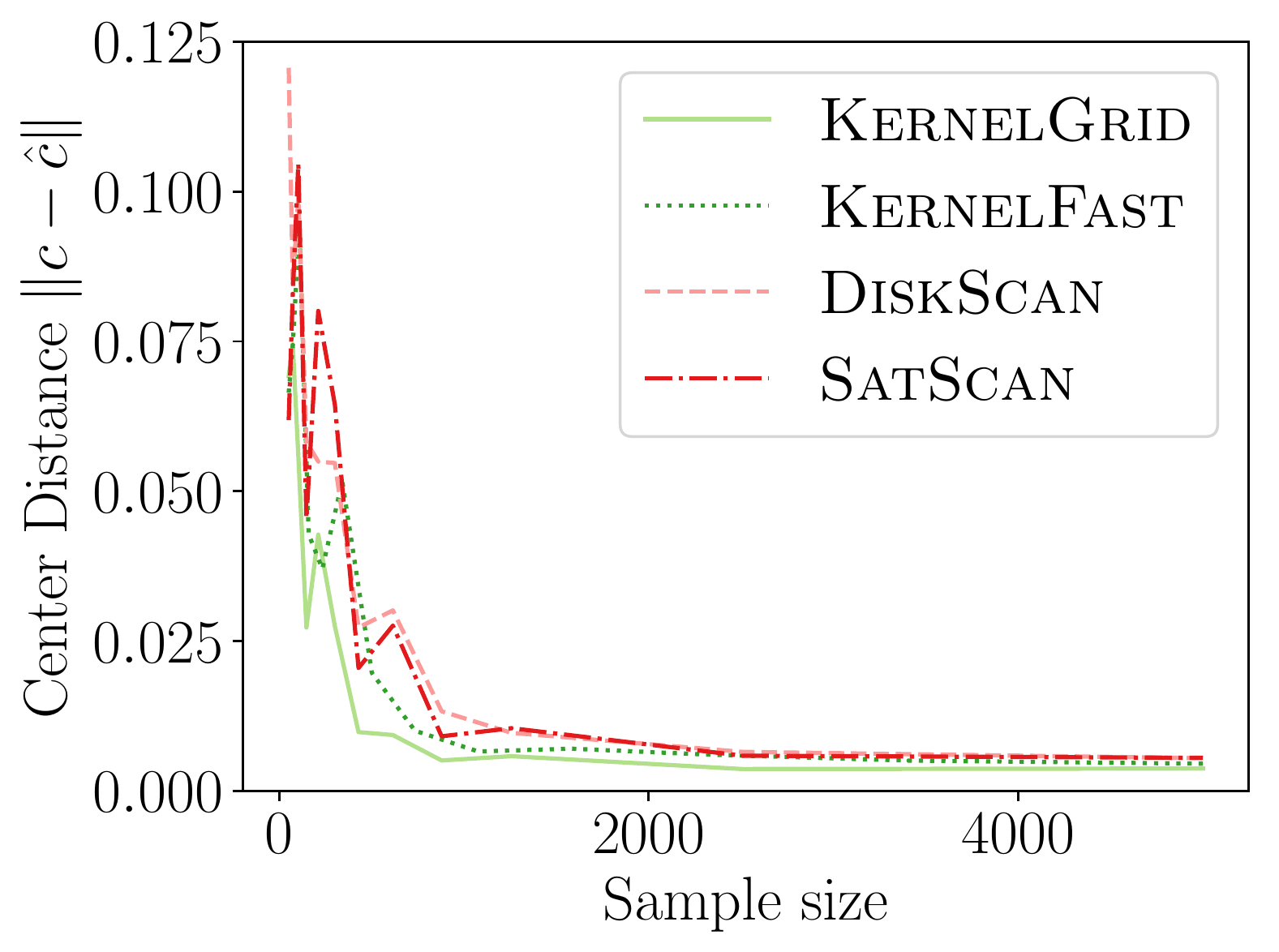}
\includegraphics[width=.3\linewidth]{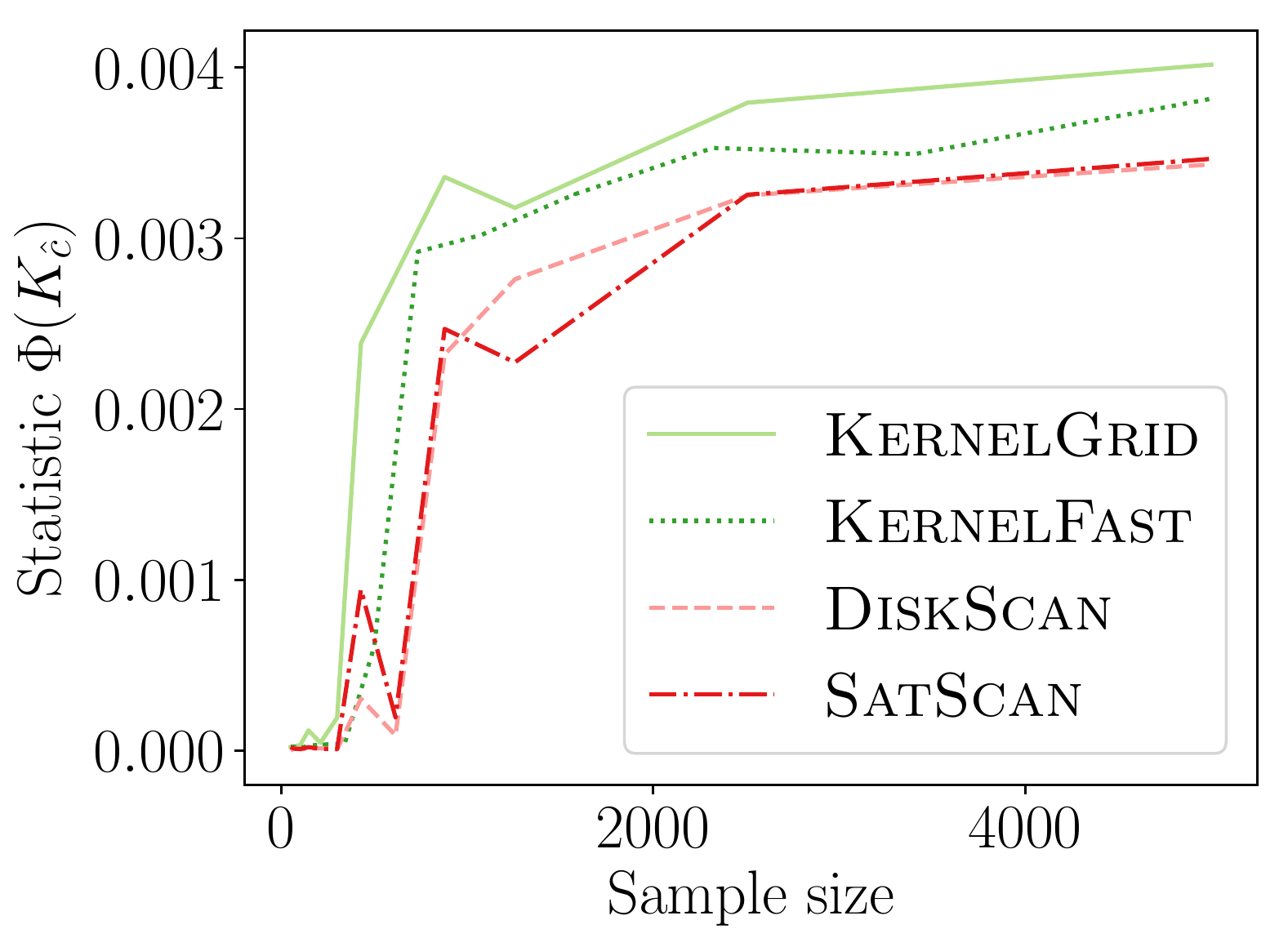}
\includegraphics[width=.3\linewidth]{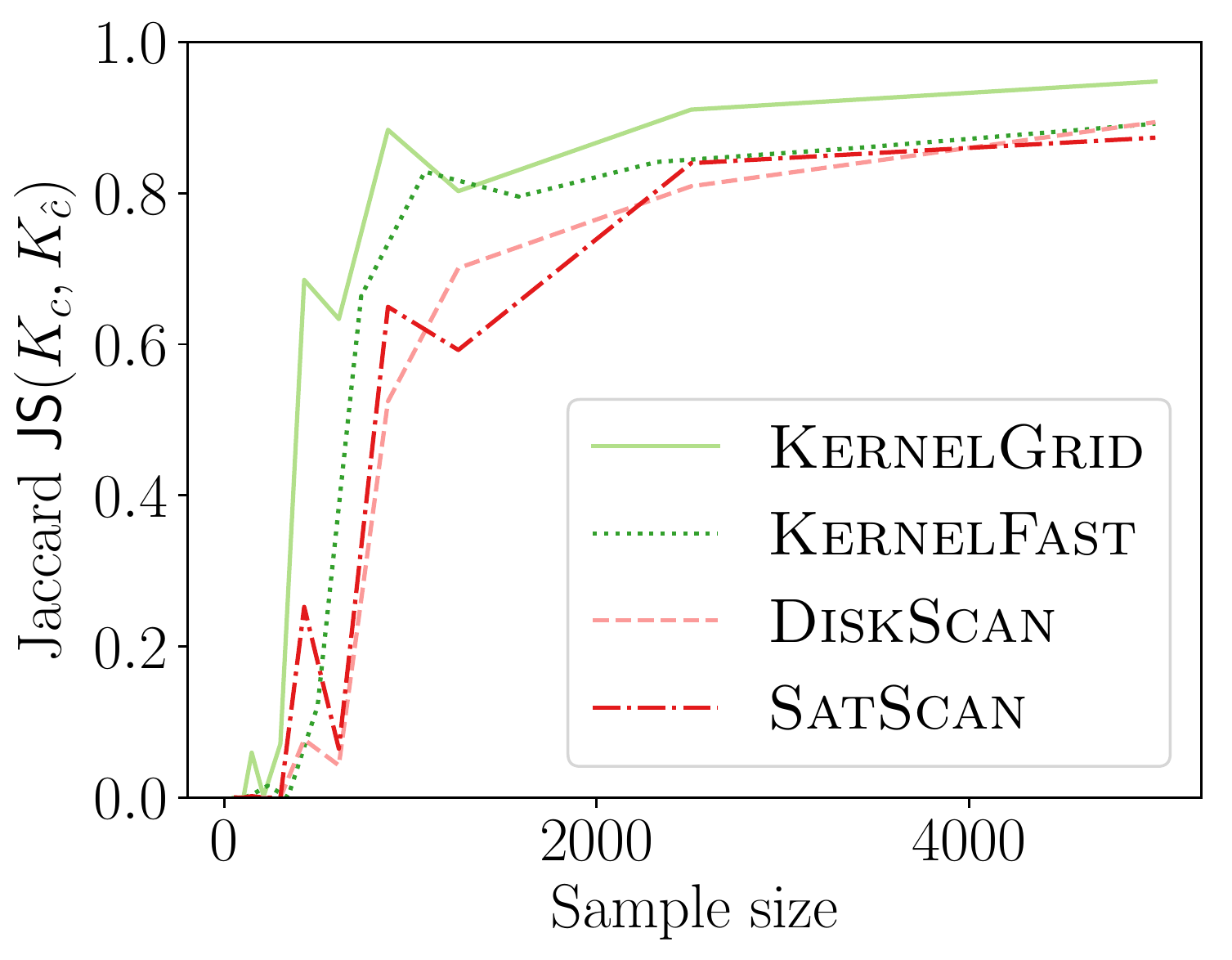}

\vspace{-5mm}
\caption{\label{fig:comparison}
New KSSS vs. Disk SSS algorithms via statistican power from sample size using $\|c-\hat c\|$, $\Phi(K_{\hat c})$, and $\mathsf{JS}(K_c, K_{\hat c})$. }
\end{figure*}

\begin{figure*}
\includegraphics[width=.3\linewidth]{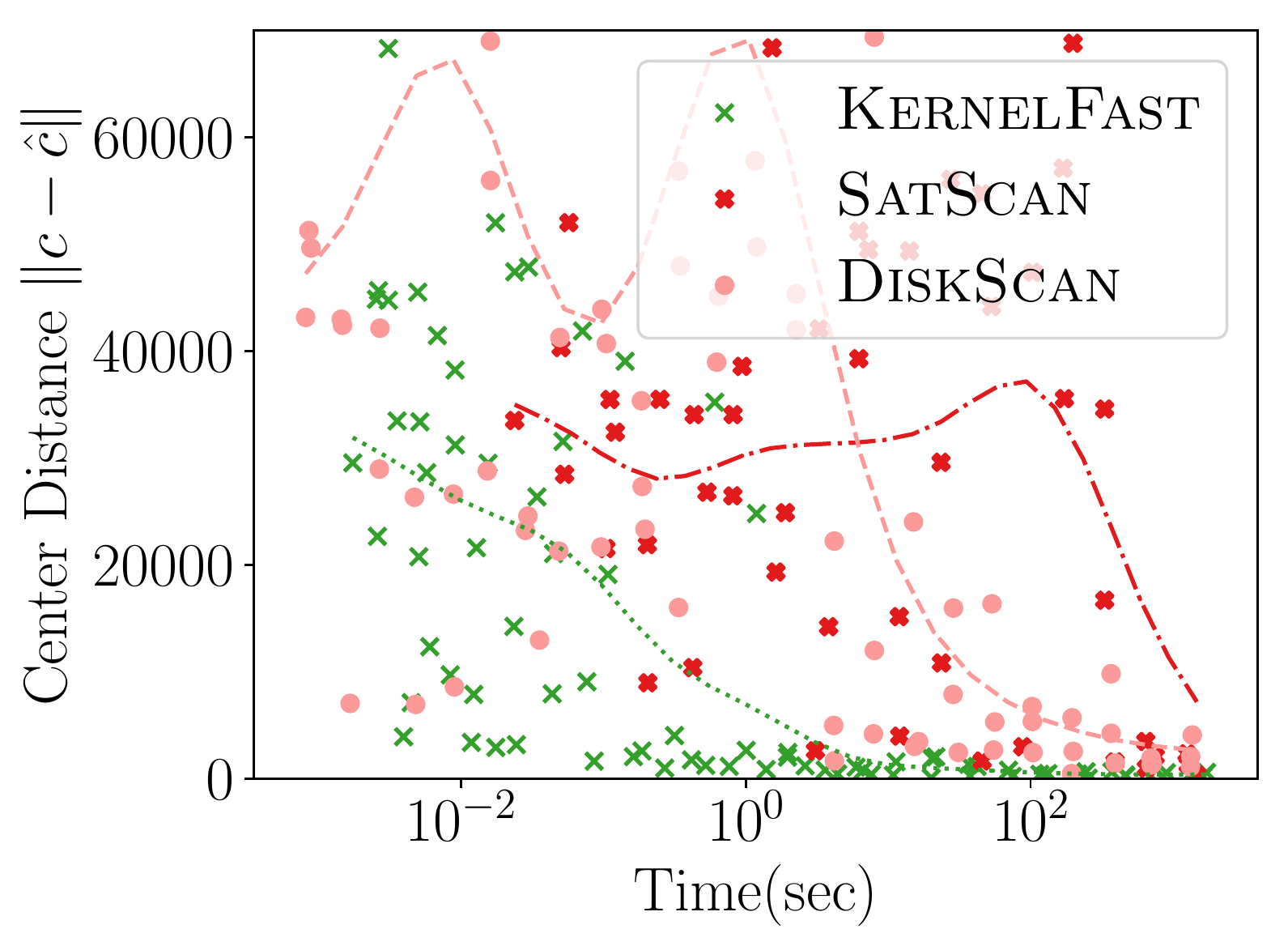}
\includegraphics[width=.3\linewidth]{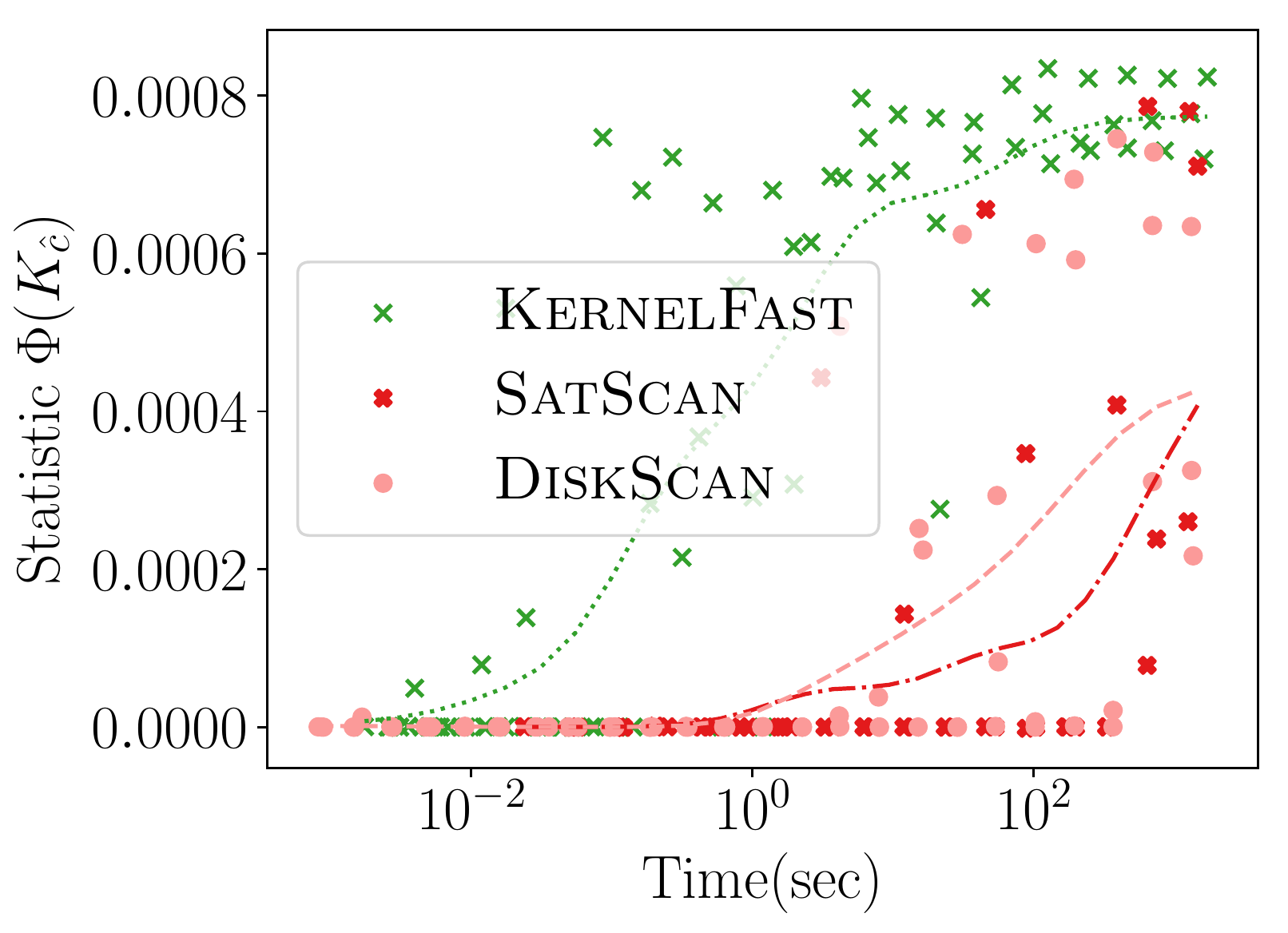}
\includegraphics[width=.3\linewidth]{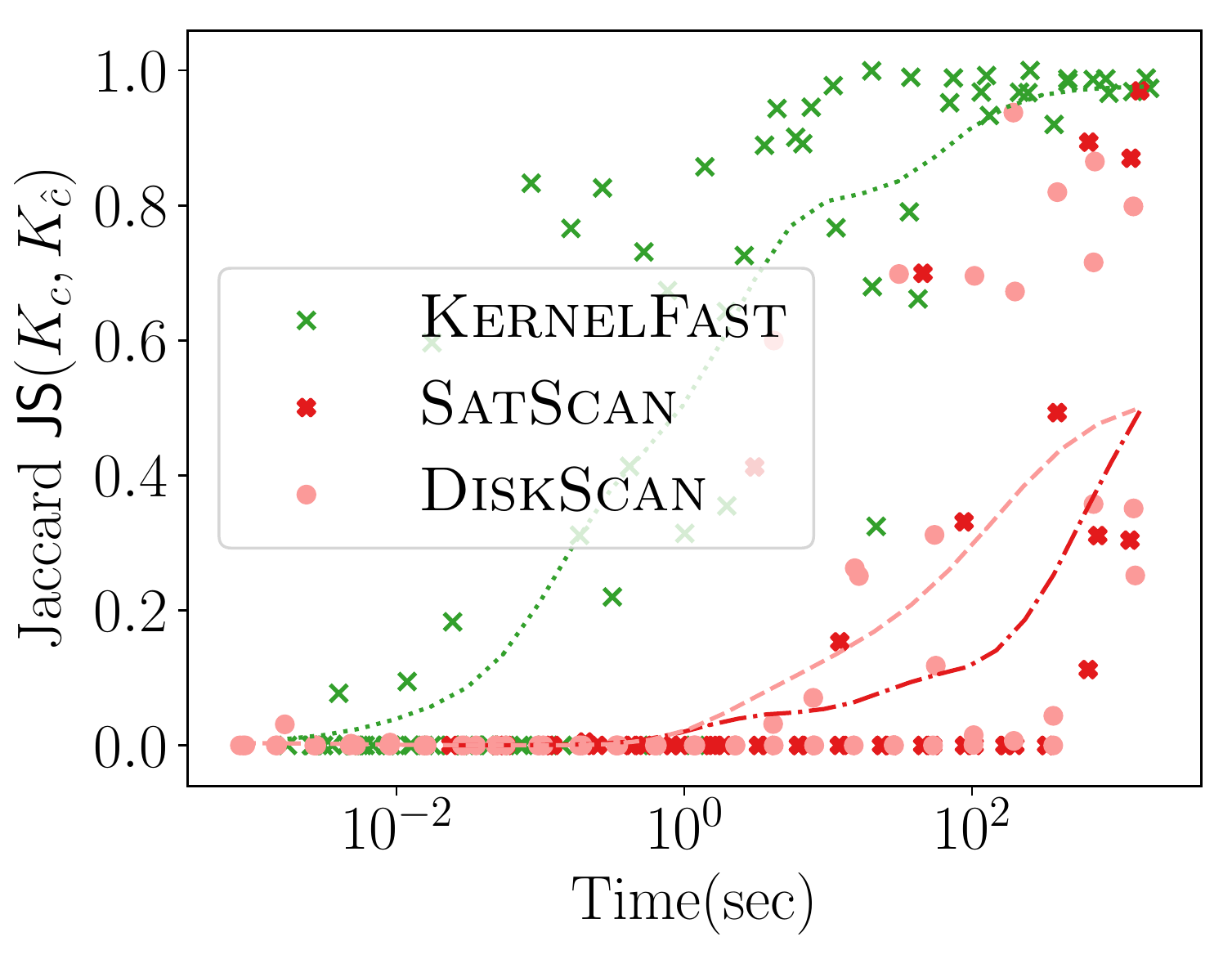}

\vspace{-5mm}
\caption{\label{fig:power-time}
Accuracy measures as a function of runtime using $\|c-\hat c\|$, $\Phi(K_{\hat c})$, and $\mathsf{JS}(K_c, K_{\hat c})$.}
\end{figure*}

\subsection{Comparing New KSSS Algorithms} 
We first compare the new KSSS algorithms against each other, as we increase the sample size $|B_\eps|$ and the corresponding other griding and pruning parameters to match the expected error $\eps$ from sample size $|B_\eps|$ as dictated in Section \ref{sec:algorithm-description}.

We observe in Figure \ref{fig:measure1-kernel} that all of the new KSSS algorithms achieve high power at about the same rate.  In particular, when the sample size reaches about $|B_\eps| = 1{,}000$, they have all plateaued near their best values, with large power:  the center distance is close to $0$, $\Phi(K_{\hat c})$ near maximum, and $\mathsf{JS}(K_c, K_{\hat c})$ almost $0.9$.  
At medium sample sizes $|B_\eps| = 500$, \textsc{KernelAdaptive} and \textsc{KernelFast} have worse accuracy, yet reach maximum power around the same sample size -- so for very small sample size, we recommend \textsc{KernelPrune}.  

In Figure \ref{fig:time-kernel} we see that the improvements from \textsc{KernelGrid} up to \textsc{KernelFast} are tremendous; a speed-up of roughly $20$x to $30$x improvement.  By considering \textsc{KernelPrune} and \textsc{KernelAdaptive} we see most of the improvement comes from the adaptive gridding, but the pruning is also important, itself adding $2$x to $3$x speed up.


\vspace{-1mm}
\subsection{Power vs. Sample Size}  
We next compare our KSSS algorithms against existing, standard Disk SSS algorithms.  As comparison, we first consider a fast reimplementation of SatScan~\cite{Kul97,Kul7.0} in C++.  To make-even the comparison, we consider the exact same center set (defined on grid $G_\eps$) for potential epicenters, and consider all possible radii of disks.  
Second, we compare against a heavily-optimized \textsc{DiskScan} algorithm~\cite{MP18b} for Disks, which chooses a very small ``net'' of points to combinatorially reduce the set of disks scanned, but still guarantee $\eps$-accuracy (in some sense similar to our adaptive approaches).  For these algorithms we maximize Kuldorff's Bernoulli likelihood function~\cite{Kul97}, whose $\log$ has a closed for over binary ranges $D \in \D$.  



Figure \ref{fig:comparison} shows the power versus sample size (representing how many data points are available), using the same metrics as before.  The KSSS algorithms perform consistently significantly better -- to see this consider a fixed $y$ value in each plot.  
For instance the KSSS algorithms reach $\|c-\hat c\| < 0.05$, $\Phi(K_{\hat c}) > 0.003$ and $\mathsf{JS}(K_c,K_{\hat c}) > 0.8$ after about 1000 data samples, whereas it takes the Disk SSS algorithms about 2500 data samples.  


\vspace{-1mm}
\subsection{Power vs. Time} 
We next measure the power as a function of the runtime of the algorithms, again the new KSSS algorithms versus the traditional Disk SSS algorithms.  We increase the sample size $|B_\eps|$ as before, now from the Chicago dataset, and adjust other error parameters in accordance to match the theoretical error.  

Figure \ref{fig:power-time} shows \textsc{KernelFast} significantly outperforms \textsc{SatScan} and \textsc{DiskScan} in these measures in orders of magnitude less time.  
It efficiently reaches small distance to the planted center faster (10 seconds vs 1000 or more seconds). In 5 seconds it achieves $\Phi^*$ of 0.0006, and 0.00075 in 100 seconds; whereas in 1000 seconds the Disk SSS only reaches 0.0004. 
Similarly for Jaccard similarity, \textsc{KernelFast} reaches 0.8 in 5 seconds, and 0.95 in 100 seconds; whereas in 1000 seconds the Disk SSS algorithms only reach 0.5.





\begin{figure}
\includegraphics[width=.49\linewidth]{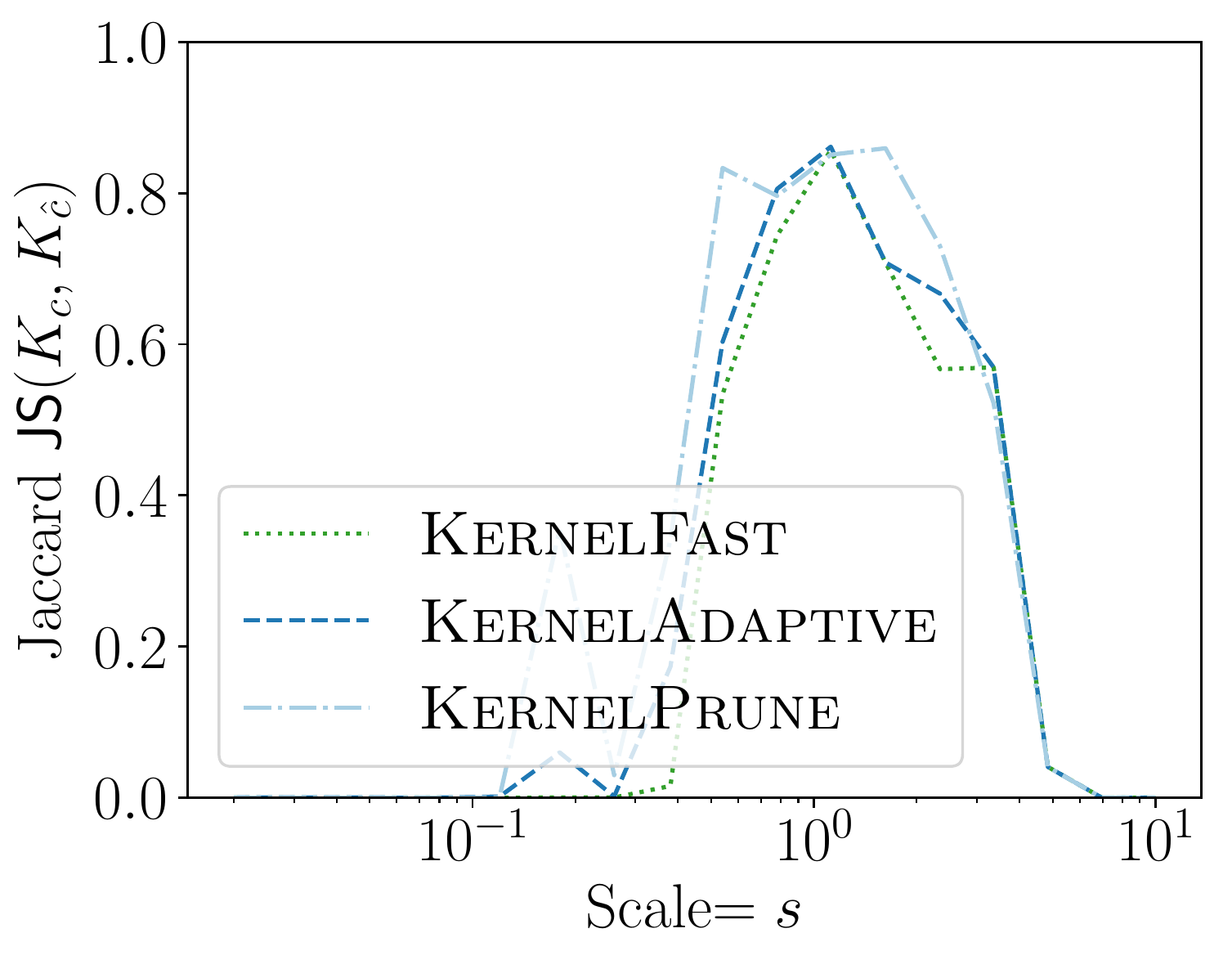}
\includegraphics[width=.49\linewidth]{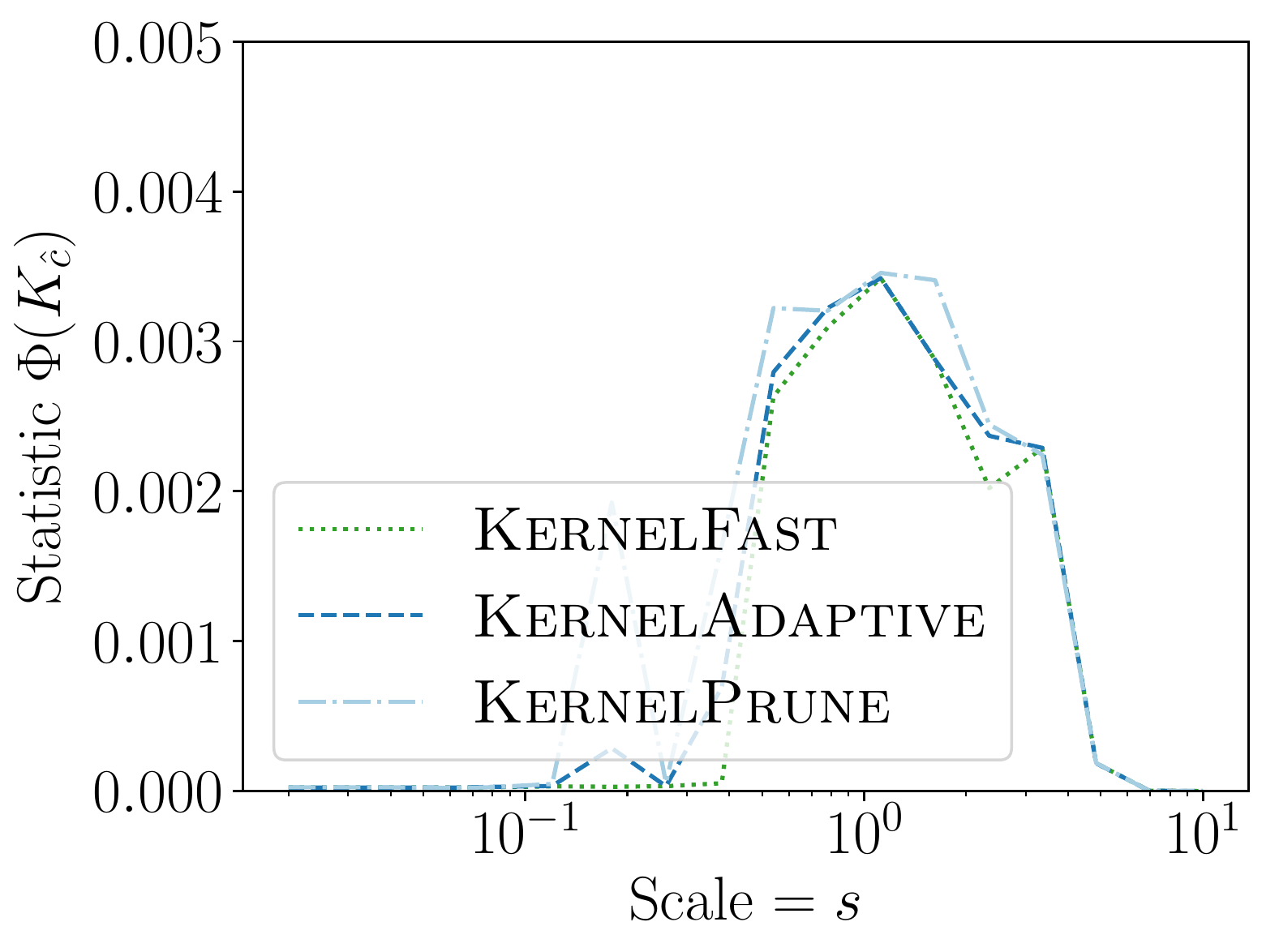}

\vspace{-5mm}
\caption{\label{fig:bandwidth-sens}
Accuracy on bandwidth $r$ of planted region.}    
\end{figure}
\subsection{Sensitivity to Bandwidth}
\label{sec:bandwidth}
So far we chose $r$ to be the bandwidth of the planted anomaly $K_c \in \K_r$ (this is natural if we know the nature of the event).  But if the true anomaly bandwidth is not known or only known in some range then our method should be insensitive to this parameter.  
On the Philadelphia dataset we consider $30$ geometrically increasing bandwidths scaled so for original bandwidth $r$ we considered $r s$ where $s \in [10^{-2}, 10]$.  
In Figure \ref{fig:bandwidth-sens} we show the accuracy using Jaccard similarity and the $\Phi$-value found, over $20$ trials.  Our KSSS algorithms are effective at fitting events with $s \in [0.5, 2]$, indicating quite a bit of lee-way in which $r$ to use.  That is, \emph{the sample complexity would not change}, but the time complexity may increase by a factor of only $2$x - $5$x if we also search over a range of $r$.  

\section{Conclusion}
In this work, we generalized the spatial scan statistic so that ranges can be more flexible in their boundary conditions.  In particular, this allows the anomalous regions to be defined by a kernel, so the anomaly is most intense at an epicenter, and its effect decays gradually moving away from that center.  However, given this new definition, it is no longer possible to define and reason about a finite number of combinatorially defined anomalous ranges.  Moreover, the log-likelihood ratio test derived do not have closed form solutions and as a result we develop new algorithmic techniques to deal with these two issues.  These new algorithms are guaranteed to approximately detect the kernel range which maximizes the new discrepancy function up to any error precision, and the runtime depends only on the error parameter.  
We also conducted controlled experiments on planted anomalies which conclusively demonstrated that our new algorithms can detect regions with few samples and in less time than the traditional disk-based combinatorial algorithms made popular by the SatScan software.  That is, we show that the newly proposed Kernel Spatial Scan Statistics theoretically and experimentally outperform the existing Spatial Scan Statistic methods.



\vspace{-2mm}
\bibliographystyle{abbrv}
\bibliography{discrepancy}


\clearpage

\appendix
\section{Approximating the Bernoulli KSSS}
\label{sec:proofs}



In the following sections, we focus on the Bernoulli model and for notational simplicity use $\Phi = \Phi^\Ber$, $\ell = \ell^\Ber$. Sometimes we also use $\Phi_{p,q}(K) = \Phi(p,q,K)$. 

The values of $p$ and $q$ are bounded between $0$ and $1$, but at these extremes $\Phi_{p,q}(K)$ can be unbounded.  
Instead we will bound structural properties of $\Phi$, by assuming $p=p^*,q=q^*$.  

\begin{lemma}
\label{lem:conservation}
When $(p=p^*,q=q^*) = \argmax_{p,q} \Phi(p,q,K)$ then
\[
    |B| = \sum_{x \in B \setminus M} \frac{1}{1 - g(x)} 
\;\;\;\; \text{ and } \;\;\;\;
    |B| = \sum_{x \in M} \frac{1}{g(x)}.  
\]
\end{lemma}
\begin{proof} 
Consider first the derivatives of the alternative hypothesis $\ell = \ell(p,q,K)$ with respect to $p$ and $q$.
\begin{align*}
    \frac{\dir \ell}{\dir p} =& \sum_{x \in M} \frac{K(x)}{g(x)} - \sum_{x \in B \setminus M} \frac{K(x)}{1 - g(x)} 
    \\
    \frac{\dir \ell}{\dir q} =& \sum_{x \in M} \frac{1 - K(x)}{g(x)} - \sum_{x \in B \setminus M} \frac{1 - K(x)}{1 - g(x)}
\end{align*}
from this we can see that at the maximum $p^*,q^*$, when $g(x) = p^*K(x) + q^*(1-K(x))$ then
\[
    \frac{\dir \ell}{\dir p}p^* + \frac{\dir \ell}{\dir q}q^* 
    = 
    \sum_{x \in M} \frac{g(x)}{g(x)} - \sum_{x \in B \setminus M} \frac{g(x)}{1 - g(x)}
    =
    0
\]
Since
$ \sum_{x \in M} \frac{g(x)}{g(x)} =  |M|$, hence $\sum_{x \in B \setminus M} \frac{g(x)}{1 - g(x)} = |M|$
and then we can derive
\[
|M| 
= 
\sum_{x \in B \setminus M} \frac{1 + g(x) - 1}{1 - g(x)} 
= 
\sum_{x \in B \setminus M} \frac{1 }{1 - g(x)} - |B \setminus M|.  
\]
As a result
\[
    |B| = \sum_{x \in B \setminus M} \frac{1}{1 - g(x)} 
\]

Similarly
\begin{align*}
\frac{\dir \ell}{\dir p}(1 - p^*) + \frac{\dir \ell}{\dir q}(1- q^*) =& \sum_{x \in M} \frac{1 - g(x)}{g(x)} - \sum_{x \in B \setminus M} \frac{1 - g(x)}{1 - g(x)}
\end{align*}
and following similar steps we ultimately find
\[
|B| = \sum_{x \in M} \frac{1}{g(x)}.  \qedhere
\]
\end{proof}

This implies simple bounds on $g(x)$ as well.  
\begin{lemma}\label{lem:gx-bounds}
When $(p=p^*,q=q^*) = \argmax_{p,q} \Phi(p,q,K)$ then
\[
     g(x) \in \left [\frac{1}{|B|}, 1 \right] \text{if $x \in M$} 
     \;\; \text{ and } \;\;
      g(x) \in \left [0, \frac{|B| - 1}{|B|} \right] \text{if $x \in B \setminus M$}
\]
\end{lemma}
\begin{proof}
From Lemma \ref{lem:conservation} we can state that 
\begin{align*}
     |B| =  \sum_{x \in M} \frac{1}{g(x)} \ge \frac{1}{g(x)} 
\end{align*}
for any $x \in M$ and therefore $g(x) \ge \frac{1}{|B|}$ for $x \in M$. In the case that $x \in B$ we have that 
\begin{align*}
     |B| =  \sum_{x \in B \setminus M} \frac{1}{1 - g(x)} \ge \frac{1}{1 - g(x)} 
\end{align*}
for any $x\in B$. Therefore $1 - g(x) \ge \frac{1}{|B|}$ or $\frac{|B| - 1}{|B|} \ge g(x)$.
\end{proof}

\subsection{Spatial Approximation of the KSSS}
\label{sec:spatial-approx}
In this section we define useful lemmas that can be used to spatially approximate the KSSS by either ignoring far away points or restricting the set of centers spatially.  

\Paragraph{Truncating Kernels}
We argue here that we can consider a simpler set of truncated kernels $\K_{r,\eps}$ in place of $\K_r$ so replacing at $K_c \in \K_r$ with a corresponding $K'_c \in \K_{r,\eps}$ without affecting $\ell(p,q,K)$ and hence $\Phi(K)$ by more than and additive $\eps$-error.  Specifically, define any $K'_c \in \K_{r,\eps}$ using $r_\text{max} = r \sqrt{\log (|B|/\eps)}$ as
\[
K'_c(x) = \begin{cases}
K_c(x) = \exp(-\|x-c\|^2/r^2) & \text{ if } \|x-c\| \leq r_\text{max}
\\
0 & \text{ otherwise.}
\end{cases}
\]

\begin{lemma}
\label{lem:truncate}
For any data set $B$, center $c \in \R^d$, and error $\eps >0$
\[
| \Phi(K_c) - \Phi(K'_c) | \leq \eps.
\]
\end{lemma}
\begin{proof}
The key observation is that for $\|x-c\| > r_\text{max}$ then $K(x) \leq \exp(-\log(|B|/\eps)) = \eps/|B|$.  Since $0 \leq q,p \leq 1$, then 
\[
|g(x) - g'(x)| \leq \eps/|B|
\]
as well; where $g'(x) = q + (p-q)K'(x)$.  

Next we rewrite $\ell(p,q,K)$ as
\[
\ell(p,q,K) 
= 
\frac{1}{|B|}\sum_{x \in B} \log \left( \begin{cases} g(x) & \text{ if } m(x)=1 \\  (1-g(x)) & \text{ if } m(x)=0 \end{cases} \right).
\]
Thus since $g(x) > (1/|B|)$ for $x\in M$ and $1 - g(x) >  (1/|B|)$ for $x \in B \setminus M$ by Lemma \ref{lem:gx-bounds}, then we can analyze each term of the average as $\log(\cdot)$ of a quantity at least $1/|B|$.  If that quantity changes by at most $\eps/|B|$, then that term changes by at most 
\[
\log(1/|B| + \eps/|B|) - \log(1/|B|) \leq (\eps / |B|) |B| = \eps,
\]
by taking the derivative of $\log$ at $1/|B|$.  
Hence each term changes by at most $\eps$, and $\ell(p,q,K)$ takes the average over these terms, so the main claim holds.  
\end{proof}
\Paragraph{Center Point Lipschitz Bound}
Next we show that $\Phi(K_c)$ is stable with respect to small changes in its epicenter $c$


\begin{lemma}
\label{lem:spatial-lipshitz}
The magnitude of the gradient with respect to the center $c$ of $\Phi(K_c)$ for any considered $K_c \in \K_r$ is bounded by 
\[
|\nabla_{c} \Phi(K_c)|  \le  (1/r) \sqrt{8/e}.  
\]
\end{lemma}
\begin{proof}
We take the derivative of $c$ in direction $\vec{u}$ (a unit vector in $\mathbb{R}^2$) at magnitude $t$.  
First analyze the Gaussian kernel under such a shift as $K_c(x) = \exp(-\| c -x + t \vec{u}\|^2/r^2)$, so as $t \to 0$ we have
\[
\left|\frac{\dir K_c(x)}{\dir t}\right| 
= 
\left|\frac{2\langle \vec{u}, (c - x) \rangle}{r^2} \exp \left (-\frac{\|c - x\|^2}{r^2} \right )\right|
\]
This maximized for $\vec{u} = (c - x) / \|c - x\|$ so
\[
\left|\frac{\dir K_c(x)}{\dir t}\right| 
\le 
\left|\frac{2(\|c - x\|)}{r^2} \exp \left (-\frac{(\|c - x\|)^2}{r^2} \right )\right|
\leq
\frac{1}{r} \sqrt{\frac{2}{e}},
\]
since it is maximized at $\|c-x\|/r = \sqrt{1/2}$.  

Now examine $\frac{\dir \Phi(K_c)}{\dir t} = \frac{\dir \ell(p^*,q^*,K_c)}{\dir t}$ which expands to  
\[
\frac{\dir \Phi(K_c)}{\dir t} 
 = \sum_{x \in B} \frac{\dir \Phi(K_c)}{\dir K_c(x)} \frac{\dir K_c(x)}{\dir t}  + \frac{\dir \Phi(K_c)}{\dir p^*}\frac{\dir p^*}{\dir t}  + \frac{\dir \Phi(K_c)}{ \dir q^*} \frac{\dir q^*}{\dir t}. 
\]

Since at $p=p^*,q=q^*$ both $\frac{\dir \Phi(K_c)}{\dir p} = 0$ and $\frac{\dir \Phi(K_c)}{ \dir q} = 0$, then as long as $\frac{\dir p^*}{\dir t}$ and $\frac{\dir q^*}{\dir t}$ are bounded the associated terms are also $0$.  
We bound these by taking the derivative of the equation
$|B| = \sum_{x \in M} \frac{1}{g(x)}$, from Lemma \ref{lem:conservation}. 
with respect to $t$:  
\[
\sum_{x \in M} \frac{1}{g(x)^2} \left ( \frac{\dir p}{\dir t}K_c(x)  + \frac{\dir q}{\dir t} (1 - K_c(x)) +   \frac{\dir K_c(x)}{\dir t} (p^* - q^*) \right ) 
= 
\frac{\dir |B|}{\dir t} = 0.
\]
The first term $\sum_{x \in M} \big|\frac{\dir K_c(x)}{\dir t}\big| \frac{ (p^* - q^*)}{g(x)^2} < C$ for some constant $C$ at $p^*$ and $q^*$ from Lemma \ref{lem:gx-bounds} and since $\frac{\dir K(x)}{\dir t}$ is bounded.
Hence
\begin{align*}
C 
& \ge  
\left | \sum_{x \in M} \frac{1}{g(x)^2} \left ( \frac{\dir p^*}{\dir t}K_c(x)  + \frac{\dir q^*}{\dir t} (1 - K_c(x)) \right ) \right | 
\\ & \ge  
\left | \frac{\dir p^*}{\dir t} \sum_{x \in M} K_c(x)  + \frac{\dir q^*}{\dir t} \sum_{x \in M} (1 - K_c(x)) \right | 
\end{align*}
Therefore if $K_c(x) > 0$ for any $x \in M$ then $\frac{\dir p^*}{\dir t}$ and $\frac{\dir p^*}{\dir t}$ are bounded. This holds for any $K_c$ considered, as otherwise $\Phi = 0$ since the alternative hypothesis $\ell^*(K_c)$ is equal to the null hypothesis $\ell^*$.  




Finally, we bound the magnitude of the first term
\begin{align*}
\left |\frac{\dir \Phi(K_c)}{\dir t} \right | 
 & = 
\left|\frac{\dir \Phi(K_c)}{\dir t} \right| \text{ for fixed } p^*,q^*
\\ & = 
\left|\frac{1}{|B|} \left(\sum_{x \in M} \hspace{-1mm} \frac{\dir \log(g(x))}{\dir t} + \hspace{-3mm} \sum_{x \in B \setminus M}\hspace{-2mm}  \frac{\dir \log(1-g(x))}{\dir t} \right) \right| \text{ fixed } p^*,q^*
\\ & =
 \frac{1}{|B|} \left |\sum_{x \in M} \frac{p^* -q^*}{g(x)}\frac{\dir K_c(x)}{\dir t} -  \hspace{-2mm}\sum_{x \in B \setminus M} \frac{p^* -q^*}{1 - g(x)}\frac{\dir K_c(x)}{\dir t} \right |
\\ &\leq
\frac{1}{|B|} \left(\frac{1}{r}\sqrt{\frac{2}{e}}\right) \left( \left|\sum_{x \in M} \frac{p^* -q^*}{g(x)}\right| +  \left|\sum_{x \in B \setminus M} \frac{p^* -q^*}{1 - g(x)}\right| \right)
\\ & =
\left(\frac{1}{r}\sqrt{\frac{8}{e}}\right), 
\end{align*}
where the last step follows by $p^*-q^* < 1$ and Lemma \ref{lem:conservation}.  
\end{proof}

This bound suggests a further improvement for centers that have few points in their neighborhood. We will state this bound in terms of a truncated kernel, but it also applies to non truncated kernels through Lemma \ref{lem:truncate}.

\begin{lemma}
\label{lem:spatial-lipshitz-adaptive}
For a truncated kernel $K'_c \in \K_{r,\eps}$, if we shift the center $c$ to $c'$ so $\|c-c'\| \leq \beta \leq r_\text{max}$, then 
\[
|\Phi(K'_c) - \Phi(K'_{c'})| \leq \beta \frac{|D \cap B|}{|B|}\frac{2 r_\text{max}}{r^2} + \eps
\]
where $D$ is a disk centered at $c$ of radius $2r_\text{max}$.  
\end{lemma}
\begin{proof}
From the previous proof for a regular kernel $K_c \in \K_r$ 
\[
\left| \frac{\dir \Phi(K_c)}{\dir t} \right| 
\leq 
 \frac{1}{|B|} \left |\sum_{x \in M} \frac{p^* -q^*}{g(x)}\frac{\dir K_c(x)}{\dir t} -  \hspace{-2mm}\sum_{x \in B \setminus M} \frac{p^* -q^*}{1 - g(x)}\frac{\dir K_c(x)}{\dir t} \right |. 
\]
Adapting to truncated kernel $K'_c \in \K_{r,\eps}$, then this magnitude times $\beta$ would provide a bound, but a bad one since the derivative is unbounded for $x$ on the boundary.  Rather, we analyze the sum over $x \in B$ in three cases.  
For $x \notin D$, these have both $K'_c(x)=0$ and $K'_{c'}(x)=0$, so do not contribute.  
For $\|x-c\| \in [r_\text{max}-\beta,r_\text{max}+\beta]$ the derivative is unbounded, but $K'_c(x) \in [0,\eps]$ and $K'_{c'}(x) \in [0, \eps]$ by Lemma \ref{lem:truncate}, so the total contribution of these terms in the average is at most $\eps$.  
What remains is to analyze the $x \in B$ such that $\|x-c\| \leq r_\text{max} - \beta$, where the derivative bound applies.  However, we will not use the same approach as in the non-truncated kernel.  

We write this contribution in terms as the sum over $M$ and $B$ separately as $\frac{\beta}{|B|}(\Delta_M - \Delta_B)$, focusing only on points in $D$ (and thus double counting ones near the bounary) where
\[
\Delta_M = \hspace{-2mm} \sum_{x \in D \cap M} \frac{p^* - q^*}{g(x)} \frac{\dir K_c(x)}{\dir t}
\; \text{ and } \;
\Delta_B = \hspace{-2mm} \sum_{x \in D \cap B \setminus M} \frac{p^* - q^*}{1-g(x)} \frac{\dir K_c(x)}{\dir t}. 
\]

To bound $\Delta_M$, observe that $g(x) = (p^*-q^*)K_c(x) + q^* \geq (p^* - q^*)K_c(x)$, and $\frac{\dir K_c(x)}{\dir t} = -2 \frac{\|x-c\|}{r^2}K_c(x)$.  Thus 
\[
|\Delta_M| 
\leq 
\left|\sum_{x \in D \cap M} \frac{p^*-q^*}{(p^* - q^*)K_c(x)} 2 K_c(x) \frac{\|x-c\|}{r^2}\right|
= 
|D \cap M| \frac{2 r_\text{max}}{r^2}.
\]

To bound $\Delta_B$, we use that $1-g(x) = 1 - q^*  - (p^* - q^*)K(x) \geq 1-q^*$, that $p^*-q^* \leq 1-q^*$, and from the previous proof that $\left| \frac{\dir K_c(x)}{\dir t} \right| \leq \frac{1}{r}\sqrt{\frac{2}{e}}$.  Thus
\[
|\Delta_B|
\leq
\left| \sum_{x \in D \cap B \setminus M} \frac{1-q^*}{1-q^*} \frac{1}{r}\sqrt{\frac{2}{e}} \right|
= 
|D \cap B \setminus M| \frac{1}{r}\sqrt{\frac{2}{e}}.
\]
Since $2 > \sqrt{2/e}$ we can combine these contributions as
\[
\frac{\beta}{|B|}|\Delta_M - \Delta_B| \leq \beta \frac{|D \cap B|}{|B|}\frac{2 r_\text{max}}{r^2}. \qedhere
\]
\end{proof}

\subsection{Bandwidth Approximations of the KSSS}
\label{sec:band-apx}
We mainly focus on solving $\max_{K_c \in \K_r} \Phi(K_c)$ for a fixed bandwidth $r$.  Here we consider the stability in the choice in $r$ in case this is not assumed, and needs to be searched over.  

\begin{lemma}
\label{lem:bandwidth}
For a fixed  $p$, $q$, and $c$ we have $\big|\frac{\dir \Phi(K)}{\dir r}\big| \le \frac{4}{e r}$.
\end{lemma}

\begin{proof}
Let $z_x = \|x - c\|_2$, $w_x = z_x/r$, then $K(x) = \exp(-w_x^2)$. 
We derive the derivative for $\ell = \ell(p,q,K)$ with respect to $r$.  
\begin{align*}
\frac{\dir \ell}{\dir r} 
&= 
\frac{1}{|B|} \left(\sum_{x \in M} \frac{2 z_x^2 (p - q)  e^{-\frac{z_x^2}{r^2}}}{r^3  g(x)} - \sum_{x \in B \setminus M} \frac{2 z_x^2 (p - q)  e^{-\frac{z_x^2}{r^2}}}{r^3  (1 - g(x))}\right)
\\ &=
\frac{2(p-q)}{r |B|} \left(\sum_{x \in M} \frac{w_x^2 e^{-w_x^2}}{g(x)} - \sum_{x \in B \setminus M} \frac{w_x^2 e^{-w_x^2}}{(1 - g(x))}\right)
\end{align*}
By two observations, $\max_{w > 0} w e^{-w} = 2/e$ and $p-q\leq 1$, simplify 
\begin{align*}
\frac{\dir \ell}{\dir r} 
&\leq 
\frac{2}{r |B|} \left(\sum_{x \in M} \frac{2/e}{g(x)} - \sum_{x \in B \setminus M} \frac{2/e}{(1 - g(x))}\right)
\end{align*}
And by Lemma \ref{lem:conservation} we bound the sums over $x \in M$ and over $x \in B \setminus M$ each by $(2/e)|B|$ and hence the absolute value by
\[
\left|\frac{\dir \ell}{\dir r}\right| \le \left|\frac{4/e}{r}\right|.  \qedhere
\]
\end{proof}

\subsection{Sampling Bounds}
\label{sec:sampling-bnds}
Sampling can be used to create a coreset of a truely massive data set $B$ while retaining the ability to approximately solve $\max_{K \in \K}\Phi(K)$.  
In particular, we consider creating an iid sample $B_\eps$ from $B$.

\begin{lemma}
\label{lem:sample-bound}
For sample $B_\eps$ of size $t = O(\frac{1}{\eps^2} \log^2 \frac{1}{\eps} \log \frac{\kappa}{\delta})$, then for $\kappa$ different centers $c \in C$, with probability at least $1-\delta$, for all estimates $\hat \ell(K_c)$ from $B_\eps$ of $\ell(K_c)$ satisfy
$
|\ell(K_c) - \hat \ell(K_c) | \leq \eps.
$
\end{lemma}
\begin{proof}
To approximate $-\ell(K_c)$ we will use a standard Chernoff-Hoeffding bound: for $t$ independent random variables $X_1, \ldots, X_t$, so $\E[X_i] = \mu$, and $X_i \in [0,\Delta_i]$ then the average $\bar X = \frac{1}{t} \sum_{i=1}^t X_i$ concentrates as
$
\Pr[|\bar X - \mu| \geq \alpha] \leq 2 \exp\left(-\frac{2\alpha^2}{t \Delta^2}\right). 
$
Set 
\[
\mu = -\frac{1}{B}\sum_{x \in B} m(x) \log(g(x)) + (1-m(x))\log(1-g(x)) = -\ell(K_c),
\]
and each $i$th sample $x' \in B_\eps$ maps to the $i$th random variable 
\[
X_i = -(m(x') \log(g(x')) + (1-m(x'))\log(1-g(x'))).
\]
Since $g(x) > 1/|B|$ for $m(x) =1$ and $1-g(x) > 1/|B|$ for $m(x)=0$ by Lemma \ref{lem:gx-bounds} we have $\Delta = -\log(1/|B|) = \log(|B|)$.  Plugging these quantities into CH bound yields
\[
\Pr[|\hat\ell(K_c) - \ell(K_c)| \leq \eps] \leq 2 \exp \left(-\frac{2 \eps^2}{t \log^2(|B|)} \right) \leq \delta'.
\]
Solving for $t$ finds
$
t = \frac{\log^2 |B|}{2 \eps^2} \log \frac{1}{2 \delta'}.
$
Then taking a union bound over $\kappa$ centers, reveals that with probability at least $1-\delta = 1-\delta'\kappa$, this holds by setting 
\[
t = \frac{\log^2 |B|}{2 \eps^2} \log \frac{\kappa}{2 \delta}.
\]

At this point, we almost have the desired bound; however, it has a $\log^2 |B|$ factor in $t$ in place of the desired $\log^2 \frac{1}{\eps}$.  One should believe that it is possible to get rid of the dependence on $|B|$, since given one sample with the above bound, the dependence on $|B|$ has been reduced to logarithmic.  We can apply the bound again, and the dependence is reduced to $\log^2 (\frac{1}{\eps}\log^2 |B|)$ and so on.  
Ultimately, it can be shown that we can directly replace $|B|$ with $\mathsf{poly}(1/\eps)$ in the bounds using a classic trick~\cite{VC71} of considering two samples $B_\eps$ and $B_\eps'$, and arguing that if they yield results close to each other, then this implies they should both be close to $B$.  We omit the standard, but slightly unintuitive details.  
\end{proof}

\subsection{Convexity in $p$ and $q$}
\label{sec:p-q-lipshitz}

We next show that $\Phi_{p,q}(K)$ is convex and has a Lipschitz bound in terms of $p$ and $q$.  However, such a bound does not exist using $p,q \in (0,1)$ as the gradient is unbounded on the boundary.  
We instead define a set of constraints related to $g(x)$ at the optimal $p^*,q^*$ that allow a Lipschitz constant.

\begin{lemma}
\label{lem:lipshitz-pq}
The following optimization problem is convex with Lipschitz constant $2|B|$, and contains $p^*,q^* = \argmin_{p,q} -\Phi_{p,q}(K)$.
\begin{align}
\begin{split}
    \min_{p, q} -\Phi_{p, q}(K)& \\
     1/B - g(x) \le 0  &\text{   for }  x \in M\\
    (|B| - 1) / |B| - g(x) \ge 0  &\text{   for  } x \in B \setminus M \\
    p \in (0, 1) \\ 
    q \in (0, 1)  \label{eq:convex-problem}
\end{split}
\end{align}
\end{lemma}
\begin{proof}
The set of constraints follow from Lemma \ref{lem:gx-bounds} and from this bound we know that the optimal $p^*$ and $q^*$ will be contained in this domain. As each constraint is a linear combination of $p$ and $q$, and hence convex, the space of solutions is also convex. 
Now since $g(x)$ is an affine transformation of $p$ and $q$ so it is both convex and concave with respect to $p,q$.  The logarithm function is a concave and non-decreasing function, hence both $\log(g(x))$ and $\log(1 - g(x))$ are concave.  The sum of these concave functions is still a concave function, and thus $-\ell(p,q,K)$ and hence $-\Phi(p,q,K)$ is convex.


The task becomes to show the absolute value of first order partial derivatives are bounded for $p$ and $q$ in this domain. We have that for $p$ (the argument for $q$ is symmetric)
\begin{align*} 
\left|\frac{\dir \ell(p,q,K)}{\delta p}\right| 
&= 
\left|\frac{1}{|B|} \left(\sum_{x \in M} \frac{K(x)}{g(x)} - \sum_{x \in B \setminus M} \frac{K(x)}{1 - g(x)}\right) \right| 
\\ & \le 
\frac{1}{|B|} \left(\left|\sum_{x \in M}\frac{K(x)}{g(x)} \right | + \left |\sum_{x \in B \setminus M}\frac{K(x)}{1 - g(x)} \right |\right) 
\\ & \le
\frac{1}{|B|}\left|\sum_{x \in M}\frac{1}{g(x)} \right | + \frac{1}{|B|} \left |\sum_{x \in B \setminus M}\frac{1}{1 - g(x)} \right | 
\\ & \leq
\frac{1}{|B|} \left |\sum_{x \in M} |B| \right | + \frac{1}{|B|}\left |\sum_{x \in B \setminus M}|B| \right | 
=
2 |B|,
\end{align*}
where the steps follow from the triangle inequality, by $K(x) \leq 1$, and that $g(x) > 1/|B|$ for $x \in M$ and $1 - g(x) > 1/ |B|$ for $x \in B \setminus M$.  
%
%
\end{proof}

\end{document}